\documentclass[11pt,english]{article}

\usepackage[margin=1in]{geometry}

\usepackage{amsmath,amsthm,mathtools,amssymb}
\usepackage{thmtools,thm-restate}
\usepackage{algorithm}
\usepackage{algorithmic}
\usepackage[usenames,svgnames,xcdraw,table]{xcolor}
\definecolor{DarkBlue}{rgb}{0.1,0.1,0.5}
\definecolor{DarkGreen}{rgb}{0.1,0.5,0.1}
\usepackage{hyperref}
\hypersetup{
   colorlinks   = true,
 linkcolor    = DarkBlue, 
 urlcolor     = DarkBlue, 
	 citecolor    = DarkGreen 
}
\usepackage[capitalize]{cleveref}
\usepackage{graphicx}
\usepackage{svg}
\usepackage{float}
\usepackage{verbatim}
\usepackage{caption}

\newcommand{\pa}{\mathrm{Pa}}
\newcommand{\qa}{\mathrm{Pa'}}
\newcommand{ \ac }{\mathrm{Ac}}
\newcommand{\Do}{\mathrm{do}}
\newcommand{\cI}{{\cal{I}}}
\newcommand{\cV}{{\cal{V}}}
\newcommand{\cU}{{\cal{U}}}
\newcommand{\cA}{{\cal{A}}}
\newcommand{\cC}{{\cal{C}}}
\newcommand{\cG}{{\cal{G}}}
\newcommand{\cL}{{\cal{L}}_\mathbf{z}}
\newcommand{\cH}{{\cal{H}}_\mathbf{z}}
\newcommand{ \zd} { \mathbf{z}}
\newcommand{ \muh} {\widehat{\mu}}
\newcommand{\cP}[3]{%
    \ifthenelse{\equal{#2}{}}{{\cal{P}}_{#3}\left(#1 \right )}{%
    \ifthenelse{\equal{#2}{}}{}{{\cal{P}}_{#3}\left(#1 \mid #2 \right) }}}

\newcommand{\cPh}[3]{%
    \ifthenelse{\equal{#2}{}}{{\widehat{\cal{P}}}_{#3}\left(#1 \right )}{%
    \ifthenelse{\equal{#2}{}}{}{{\widehat{\cal{P}}}_{#3}\left(#1 \mid #2 \right) }}}

\newtheorem{theorem}{Theorem}
\newtheorem{lemma}{Lemma}
\newtheorem{claim}[lemma]{Claim}

\newtheorem{definition}{Definition}

\DeclareMathOperator*{\argmax}{arg\,max}

\newcommand{\A}{\mathcal{A}}

\renewcommand{\O}{\mathcal{O}}

\newcommand{\Prob}{\mathbb{P}}
\newcommand{\E}{\mathbb{E}}
\newcommand{\prob}{\mathbb{P}}
\newcommand{\indic}{\mathbb{I}}

\newcommand{\Rg}{\textsc{R}}

\newcommand{\CI}{\textsc{CoveringInterventions}}

\newcommand{\DE}{\textsc{DirectExploration}}
\newcommand{\PI}{\textsc{PropInf}}

\newcommand{\x}{\mathbf x}

\title{\bfseries Learning Good Interventions in Causal Graphs via Covering}

\author{Ayush Sawarni\thanks{Indian Institute of Science. {\tt ayushsawarni@iisc.ac.in}}  \quad Rahul Madhavan\thanks{Indian Institute of Science. {\tt mrahul@iisc.ac.in}} \quad Gaurav Sinha\thanks{Microsoft Research, Bangalore. {\tt gauravsinha@microsoft.com}} \quad Siddharth Barman\thanks{Indian Institute of Science. {\tt barman@iisc.ac.in}}}
\date{\empty}

\begin{document}
\maketitle
\begin{abstract}
We study the causal bandit problem that entails identifying a near-optimal intervention from a specified set $\cA$ of (possibly non-atomic) interventions over a given causal graph. Here, an optimal intervention in $\cA$ is one that maximizes the expected value for a designated reward variable in the graph, and we use the standard notion of simple regret to quantify near optimality. 
Considering Bernoulli random variables and for causal graphs on $N$ vertices with constant in-degree, prior work has achieved a worst case guarantee of $\widetilde{O} (N/\sqrt{T})$ for simple regret. The current work utilizes the idea of covering interventions (which are not necessarily contained within $\cA$) and establishes a simple regret guarantee of $\widetilde{O}(\sqrt{N/T})$. Notably, and in contrast to prior work, our simple regret bound depends only on explicit parameters of the problem instance. We also go beyond prior work and achieve a simple regret guarantee for causal graphs with unobserved variables. 
Further, we perform experiments to show improvements over baselines in this setting.
\end{abstract}

\section{Introduction}
\label{sec: Introduction}

Causal Bayesian Networks (CBNs) are a prominent paradigm for modelling many real world problems \cite{pearl2009causality}. Recent applications include language modelling \cite{sevilla2020explaining}, medicine \cite{koch2017causal,caillet2015hip,lee2018reducing}, robotics \cite{yoshida2012active} and computational advertising \cite{bottou2013counterfactual}.

While CBNs have been the focus of research for decades, questions related to online learning in the CBN context have been studied only recently. Prototypical questions at the interface of online learning and CBNs are captured by the causal bandits model. Causal bandits---first introduced by Lattimore et al.~\cite{lattimore2016causal}---merges concepts from CBNs and multi-armed bandits (MABs) to provide a framework for optimized learning over CBNs. The focus of the current work is to obtain simple regret guarantees in the causal bandit setup. 

A CBN consists of a causal graph---a directed acyclic graph $\cG = (\cV,E)$---that provides the direction of causation amongst $N \coloneqq |\cV|$ random variables. That is, in the given graph $\cG$, the vertices $\cV$ correspond to variables and $E$ corresponds to the set of (directed edges) causal relations between these variables. Here, each variable is some function of its parents. Complementarily, a variable that has no parents (known as an exogenous or independent variable) is a random variable over some distribution; see \cite{pearl2009causality} for a textbook treatment of CBNs. 

We will, throughout, consider $\cV$ to be Bernoulli random variables. In the causal bandit problem, one designates a particular vertex in the causal graph $V_N \in\cV$ as the reward variable and seeks to optimize for the expected value of this reward variable. The optimization is over a specified set $\cA$, which consists of {interventions} in the causal graph. Interventions, known as $\Do()$ operations, fix the values of certain variables, irrespective of their parents. Specifically, in an intervention $A = \Do(S=s)$, we fix the value of each variable $i \in S \subseteq \cV$ to be the $i$th component of the given binary assignment $s \in \{0,1\}^{|S|}$. Under intervention $A = \Do(S=s)$, the un-intervened variables (in $\cV\setminus S$) then follow the causal relations that remain. The goal of the causal bandit learner is to perform exploratory interventions, for a given number of rounds $T$, and at the end of this time horizon, the learner needs to identify a near-optimal intervention from within the target set $\cA$. That is, the overarching objective in the causal bandit problem is to identify an intervention $A \in \cA$ under which the expected value of the reward variable, $V_N$, is as high as possible.

As in the classic multi-armed bandits literature \cite{lattimore2020bandit,slivkins2019introduction}, the notion of simple regret is used to quantify near optimality in the causal bandit setup. In particular, for an algorithm that selects intervention $A_T \in \cA$ after after $T$ rounds, the simple regret is the difference (in expectation) between optimal reward and the reward induced by $A_T$. 

Most prior works on causal bandits   \cite{lattimore2016causal,maiti2022causal,sen2017identifying,lu2020regret,nair2021budgeted,sen2017contextual,lu2021causal,lu2022efficient} address the problem with $\cA$ restricted to atomic interventions. That is, these works hold when each intervention $A \in \cA$ fixes some single vertex in the causal graph. Other causal bandit results \cite{varici2022causal,xiongcombinatorial} consider settings in which all the causal relations in $\cG$ are confined to be linear functions. In this active thread of research on causal bandits, a notable exception is the work of Yabe et al.~\cite{yabe2018causal}, which addresses the broad setting of non-atomic interventions over general graphs and holds without assumptions on the causal relations. 

Indeed, such a general form of the problem is nontrivial. In particular, the number of non-atomic interventions under consideration can be exponential in the number of variables $N$. Hence, a naive approach of sampling for each intervention $A \in \cA$ can yield a simple regret proportional to $\Tilde{\O}(\sqrt{\exp(N)/T})$. Interestingly, for the general form of the causal bandit problem, Yabe et al.~\cite{yabe2018causal} achieve a worst-case guarantee on simple regret of $\Tilde{\O}(\sqrt{N^2/T})$; here, the $\Tilde{\O}$ notation subsumes the dependence on the maximum in-degree in the causal graph and logarithmic factors. In particular, the regret guarantee of \cite{yabe2018causal} depends on the optimal value of a proposed optimization problem. We also note that, even during exploration, Yabe et al.~\cite{yabe2018causal} consider interventions only  from the target set $\cA$ and use the solution of the proposed optimization problem to guide the exploration. 

While the algorithm of Yabe et al.~\cite{yabe2018causal} is applicable with significant generality, it has certain key limitations. Firstly, the algorithm entails solving a non-convex optimization problem that is ``time-consuming to solve'' (see page 8 in Section 5 of \cite{yabe2018causal}).
In fact their own experiment implementations do not explicitly solve the optimization problem. Next, the regret bound is in terms of a quantity that is analytically unwieldy to estimate. In particular, their simple regret guarantee is ${O}\left( \sqrt{\frac{\gamma^* \log \left( |\cA| T \right)}{T}} \right)$, where $\gamma^*$ is the optimal value of a (hard to compute) non-convex optimization problem and it satisfies $\gamma^* = O(N^2)$. In addition, the regret guarantee in \cite{yabe2018causal} holds for time horizon $T\gtrsim N^{16}$.
Finally, their algorithm expects full observability, and does not allow for the presence of unobserved (hidden) variables in the causal graph.

The current work develops an algorithm that overcomes the above-mentioned limitations and continues to address the general form of the causal bandit problem. We use the idea of covering interventions and improve the simple regret guarantee. We also go beyond \cite{yabe2018causal} and achieve a simple regret guarantee for causal graphs with unobserved variables. 

\subsection{Our Contributions}
\label{sec: our contributions}
We present an algorithm to minimize simple regret in the causal bandit problem. Here, the learner is given a causal graph $\cG$ on $N$ Bernoulli random variables and a set $\cA$ of (possibly non-atomic) interventions over $\cG$. The learner's objective is to identify, within $\cA$, an intervention that maximizes the expected value for a designated reward variable in $\cG$. Furthermore, we consider a model wherein, while a near-optimal intervention is required from the target set $\cA$, the learner is not confined to $\cA$ during the exploration phase. In particular, we use the construct of covering interventions (see  \hyperref[definition: CIS]{Definition \ref{definition: CIS}}) during exploration and show that this flexibility leads to multiple improvements over prior work. Indeed, this model is applicable in many settings wherein the learner is not confined to the target set  during exploration. Consider, as stylized examples: (i) the display advertising context, wherein, during testing, one can intervene upon features, which during deployment, cannot be altered, and (ii) robotic control, in which, during simulations, hypothetical configurations can be deployed.    

In fact, our result is robust enough to be used in settings where certain variables cannot be intervened upon even during exploration. One can consider such `off-limits' variables as unobserved and then utilize our extension to graphs with unobserved parts (see Section \ref{section: Algorithm for SMBN}). The list below summarizes our main contributions: 

\begin{itemize}
\item  For the causal bandit problem, we improve the worst-case guarantee for simple regret from $\Tilde{\O}(\sqrt{N^2/T})$ to $\Tilde{\O}(\sqrt{N/T})$.\footnote{As mentioned previously, $T$ denotes the time horizon (i.e., number of exploratory interventions) and $N$ denotes the number of vertices in the causal graph.} Here, the $\Tilde{\O}(\cdot)$ notation subsumes the dependence on the maximum in-degree $d$ in the graph and logarithmic factors; see Theorem \ref{thm: main theorem full obs} for an explicit bound. Our algorithm can address arbitrary causal graphs. Though,  as in prior works \cite{yabe2018causal, acharya2018learning}, our result is particularly relevant for graphs in which the maximum in-degree $d$ is sufficiently smaller than $N$. 

\item We obtain a novel simple regret algorithm for causal graphs with unobserved  variables. This extension addresses the most general setting for causal Bayesian networks (see Definition 1.3.1 in \cite{pearl2000models}) and addresses a key limitation of almost all\footnote{The exceptions here are the recent works of Maiti et al.~\cite{maiti2022causal} along with Xiong and Chen~\cite{xiongcombinatorial}. These works are discussed at the end of the section.} prior works on causal bandits. We detail the extension in Section \ref{section: Algorithm for SMBN}.

\item Our experiments show a marked improvement on the baselines from prior work (see Section \ref{sec: experimental details}), thereby substantiating the theoretical guarantees.
\end{itemize}

Our worst-case guarantee for simple regret is in terms of only the explicit parameters, such as the number of variables $N$ and the maximum in-degree in $\cG$; see Theorem \ref{thm: main theorem full obs}. By contrast, the simple regret bound provided in \cite{yabe2018causal} depends on analytically complex quantities. In addition, our guarantee holds for time horizon $T\gtrsim N^3$. This is a marked improvement over \cite{yabe2018causal}, which requires $T\gtrsim N^{16}$. In fact, our algorithm (Algorithm \ref{algo: main algo fully obs setting}) is notably simple -- we view this as a positive feature, which aids in implementation and adaptation of the developed method. Here, it is also relevant to note that the key technical contribution of the work is the involved regret analysis (see Section \ref{section:regret-analysis}).

\subsection{Additional Related Work}
\label{sec: related work}
Lattimore et al.~\cite{lattimore2016causal} first addressed the causal bandit, though only for parallel causal graphs and with atomic interventions. Maiti et al.~\cite{maiti2022causal} extended this work on atomic interventions to provide simple regret guarantees in the presence of unobserved or hidden variables. An importance sampling based approach was studied in \cite{sen2017identifying} to identify atomic soft interventions that minimize simple regret. Lu et al.~\cite{lu2020regret} provide guarantees for cumulative regret for general causal graphs (which include hidden variables). Nair et al.~\cite{nair2021budgeted} looked at cumulative as well as simple regret in case of the budgeted setting where the observation-intervention trade-off was studied when interventions are costlier than observations. Sen et al.~\cite{sen2017contextual} extend the model causal bandits to include contextual causal bandits and study cumulative regret in this context. Lu et al.~\cite{lu2021causal} study cumulative regret in the case where the full graph structure is not known. The work \cite{lu2022efficient} extends the model for causal bandits to include causal Markov decision processes (C-MDPs) using a modification of the algorithm in \cite{azar2017minimax}. 

There are two recent works that focus on non-atomic interventions in the causal bandit context. The paper by Varici et al.~\cite{varici2022causal} studies cumulative regret for causal bandits with non-atomic interventions, albeit in the specific context of linear structural equation models. Xiong and Chen \cite{xiongcombinatorial} obtain sample-complexity bounds for identification of near-optimal interventions, with a particular focus on binary generalized linear models (BGLMs). The worst-case sample complexity guarantee obtained in \cite{xiongcombinatorial} is proportional to the size of the intervention set $\cA$, i.e., proportional to $|\cA|$. By contrast, the simple regret bound obtained in the current work has only a logarithmic dependence on $|\cA|$; recall that $|\cA|$ can be exponentially large. Xiong and Chen \cite{xiongcombinatorial} also address the case of unobserved (hidden) variables. However, this work assumes identifiability (the fact that all interventional distributions can be estimated through observations alone). We require no such assumption. 

Apart from these works on causal bandits, we utilize the idea of covering interventions proposed by Acharya et al.~\cite{acharya2018learning}. They use covering interventions for distribution learning and testing problems over causal graphs. On the other hand, we use covering interventions for simple regret minimization. It is important to note that a direct use of the distribution learning algorithm (Algorithm 3) from \cite{acharya2018learning} leads to a suboptimal regret bound for the causal bandit problem. Specifically, the learning algorithm of Acharya et al.~\cite{acharya2018learning} requires $\Tilde{\O}(N^2\varepsilon^{-4})$ samples to learn interventional distributions up to a total variation distance of $\varepsilon$; see Theorem 3.4 in \cite{acharya2018learning}. Hence, if used for identifying a near-optimal  intervention in $\cA$, this method would incur $\Tilde{\O} \left(\frac{\sqrt{N}}{T^{1/4}} \right)$ simple regret. 

\section{Notation and Preliminaries}
\label{sec: notation and problem setup}
We study the causal bandit problem over causal graphs $\cG = (\cV,E)$. In the given (directed and acyclic) graph $\cG$ the vertices, $\cV$, correspond to Bernoulli random variables and $E$ is the set of directed edges that capture causal relations between these variables. 

We will use $V_i$ or $i$, interchangeably, to refer to the $i$th node of the given causal graph $\cG$. Since $\cG$ is directed and acyclic, it admits a topological ordering. We will, throughout, assume that the vertices in $\cV$ are indexed to respect a topological order, i.e., for each pair of indices $i < j $, vertex $V_i$ appears before $V_j$ in the topological order. Note that for any subset of vertices $\cU \subseteq \cV$ the indexing of the vertices within $\cU$ follows the topological ordering of these vertices. Furthermore, in the set $\cV$, the last vertex with respect to the indexing (and, equivalently, the topological ordering) is the designated reward variable. That is, in a causal graph with $N \coloneqq |\cV|$ vertices, $V_N$ is the reward variable.  

Write $\pa (i)$ to denote the set of parents of node $V_i$. Also, we define the set of parents for a subset of vertices  $\cU \subseteq \cV $ as $\pa (\cU) \coloneqq  \left(\cup_{  V \in \cU } \ \pa(V) \right) \setminus  \cU $. We use the following notations to indicate subsets of the vertices: write $[i,j] \coloneqq  \{V_i, V_{i+1}, V_{i+2} \ldots V_j\}$ and, similarly, $(i,j] = [i+1,j]$, $(i,j)= [i+1,j-1]$ and $[i,j) = [i, j-1]$. Write the ancestor set $\ac (i) \coloneqq [1,i) \setminus \pa(i)$, i.e., $\ac(i)$ denotes the set of vertices that precede $V_i$ in the topological ordering, excluding the parents $\pa(V_i)$.

An intervention is defined as an $N = |\cV|$ dimensional vector $A \in \{0,1,*\}^{N}$ that encapsulates the values assigned to each vertex in $\cG$; in particular, $A_i = * $ denotes that $V_i$ is not intervened upon, while $A_i = 1$ and $A_i =0$ denote that, in the intervention, $V_i$ is set to $1$ and $0$, respectively. In addition, $\cV (A) \coloneqq  \{V_i \in \cV : A_i = * \}$  denotes the set of vertices that are not intervened under $A$. Also, for any subset of vertices $\cU \subseteq \cV$, write $\cV_{\cU}(A) \coloneqq \cU \cap \cV(A)$. 

Binary vectors $\zd \in \{0,1\}^N$ will be used to denote an assignment to the vertices (random variables) in $\cV $. Here, $\zd_i $ denotes the assignment to vertex $V_i$. For any subset of vertices $\cU \subseteq \cV$, we will use $\zd_{U} \in \{0,1\}^{|\cU|}$ to denote an assignment to the vertices in $\cU$. Let $Z(A)$ denote the set of all binary assignments that comply with an intervention $A$ and have the reward $V_N=1$, i.e., $
Z(A) \coloneqq \{\zd \in \{0,1\}^N :  \zd_i = A_i$, for all $ i \in \cV \setminus \left( \cV (A)\right)$, and $  \zd_N =1\}$.

We use the following short-hand notations in our analysis to denote the conditional and interventional probability distributions:
 \begin{align*}
    \cP{\zd_i}{\zd_U}{} &= \prob \left[ V_i = \zd_i | U = \zd_U\right]. \\
    \cP{\zd_i}{}{\zd_U} &= \prob \left[ V_i = \zd_i |  \Do \left (U = \zd_U\right ) \right] \\
    &= \prob_{\Do(U = \zd_U)} \left[ V_i = \zd_i \right]. \\
    \cP{\zd_i}{\zd_W}{\zd_U} &= \prob \left[ V_i = \zd_i |  \Do \left (U = \zd_U\right ), W= \zd_W  \right]. \\
    \cP{\zd_i}{\zd_W}{A} &= \prob \left[ V_i = \zd_i |  \Do \left (A\right ), W= \zd_W  \right]. \\
 \end{align*}
It is important to note that intervening on all parent nodes of a vertex is the same as conditioning on them
\begin{align}
    \cP{\zd_i}{}{\zd_{\pa (i)}} = \cP{\zd_i}{\zd_{\pa (i)}}{} \label{eqn:all-parent}
\end{align}

We use $\mu\left( A \right)$ to denote the expected reward under intervention $A$, i.e., $\mu \left(A \right) = \prob \left[ V_N = 1 | \Do(A) \right]$. Specifically, 
$$\mu \left( A\right) = \sum_{\zd \in Z(A)} \prod_{ i \in \cV(A)} \cP{\zd_i}{\zd_{\pa(i)}}{}$$
We use $ \widehat{\mu}(A)$ and $\cPh{}{}{}$ to denote the estimates for the corresponding quantities, and $ \Delta \cP{}{}{} $ to denote the error in the estimates. In particular, for an empirical estimation in which vertex $V_i$ is sampled $T_i$ times, with parents taking value $\zd_{\pa(i)} \in \{0,1\}^{|\pa(i)|}$, we have estimate
\begin{align*}
    \cPh{\zd_i}{\zd_{\pa(i)}}{} = \frac{\sum_{s=1}^{T_i} \indic [ Y_{i,s}  = \zd_i ] }{T_i}, 
\end{align*}
where $Y_{i,s}$ is the s-{th} sample of vertex $V_i$. In addition, we have 
\begin{align}
    \Delta \cP{\zd_i}{\zd_{\pa(i)}}{} & = \cP{\zd_i}{\zd_{\pa(i)}}{} - \cPh{\zd_i}{\zd_{\pa(i)}}{} \nonumber\\ 
    \widehat{\mu}(A) & =  \sum_{\zd \in Z(A)} \prod_{ i \in \cV(A)} \cPh{\zd_i}{\zd_{\pa(i)}}{} \label{eq:calc_emperical} 
\end{align}

Recall that, in the causal bandits problem, the objective is to find---from within a specified collection of interventions $\cA$---an intervention with maximum possible expected reward. We will write $A^* \in \cA$ to denote the optimal intervention and $\mu(A^*)$ for the optimal reward, i.e., $\mu(A^*) = \max_{A \in \cA} \ \mu(A)$. Also,  for any algorithm, let $A_T \in \cA$ be the (randomized) output computed after $T$ rounds; in each round the algorithm performs an intervention and observes a sample under it.\footnote{Note that while the computed intervention must be contained in set $\cA$, the interventions performed in the $T$ rounds are not necessarily from $\cA$.} The simple regret of the algorithm is defined as 
\begin{equation}
    R_T = \E \left[ \mu(A^*) - \mu(A_T) \right].
\end{equation}

\section{Finding Near-Optimal Intervention via Covering}
\label{sec: main algorithm} 
To find a near-optimal intervention from the given set of interventions $\cA$ (specifically, to bound simple regret), instead of directly performing each $A \in \cA$, we utilize  interventions from a curated set of interventions $\cI$, referred to as the covering intervention set (see Definition \ref{definition: CIS}). The obtained samples are then used to estimate the interventional distribution for each $A \in \cA$ and, hence, find a near-optimal intervention within $\cA$. The notion of covering intervention set was formulated in \cite{acharya2018learning} and is defined next. 

\begin{definition}[Covering Intervention Set] \label{definition: CIS}
 A collection of interventions $\cI$ is said to be a \emph{covering intervention set} iff, for each vertex $i \in \cV$ and every assignment $\zd_{\pa(i)}  \in \{0,1\}^{|\pa(i)|}$, there exists an intervention $I \in \cI$ with the properties that 
  \begin{itemize}
  \item Vertex $i$ not intervened in $I$ (i.e., $I_i = *$).
   \item  Every vertex in $\pa(i)$ is intervened (i.e., $I_p \neq *$, for all $p \in \pa(i)$).
   \item $I$ restricted to $\pa(i)$ has the assignment $\zd_{\pa(i)}$ (i.e., $I_p = \zd_{\pa(i),p}$ for all $p \in \pa(i)$).
 \end{itemize}
\end{definition}

It is shown in \cite{acharya2018learning} that, for any causal graph $\cG$ with $N$ vertices and in-degree at most $d$, one can construct---using a randomized method---a covering intervention set $\cI$ of size $O\left( d \ 2^d \log (NT) \right)$. 

Specifically, for count $k = 3d \  2^d (\log N +2d + \log T )$, one can populate $k$ interventions $I \in \{0,1,*\}^N$ as follows: for each variable $i \in \mathcal{V}$, independently, set
             $$
            I_i = \begin{cases} 0 & \text{ with probability }  \frac{d}{2(1+d)}, \\ 
            1 & \text{ with probability }  \frac{d}{2(1+d)}, \\
            * & \text{otherwise.}
            \end{cases}
             $$   
All the constructed $k$ interventions constitute the set $\cI$. This randomized construction is known to succeed (in providing a covering interventions set) with probability at least $\left(1-1/T\right)$. Formally,\footnote{This lemma is a direct implication of Lemma 4.1 from \cite{acharya2018learning}, instantiated with $\delta = \frac{1}{T}$, $K =2$.}  
             
\begin{lemma}[\cite{acharya2018learning}]  \label{lem:covers}
For any moderately large $T \in \mathbb{Z}_+$, every causal graph $\cG$---with $N$ vertices and in-degree at most $d$---admits a covering intervention set $\cI$ of size $k = 3d \  2^d (\log N +2d + \log T )$. Furthermore, such a set $\cI$ can be found with probability at least $(1- {1}/{T})$.
\end{lemma}
We will write $\textsc{ConstructCover}(\cG)$ to denote the randomized construction of $\cI$ mentioned above. $\textsc{ConstructCover}(\cdot)$ will be used as a subroutine in our simple-regret algorithm (Algorithm \ref{algo: main algo fully obs setting}).

Theorem \ref{thm: main theorem full obs}, stated below, is the main result of this section. The theorem asserts that, for causal graphs with constant in-degree and $N$ vertices, Algorithm \ref{algo: main algo fully obs setting} achieves a simple regret of $\widetilde{O} \left( \sqrt{{N}/{T}} \right)$. 

Given a causal graph $\cG$ and a collection of interventions $\cA$, Algorithm \ref{algo: main algo fully obs setting} first obtains a covering intervention set $\cI$, for the graph $\cG$, via the subroutine \textsc{ConstructCover}. Then, the algorithms performs, $T/|\cI|$ times, each intervention $I \in \cI$. Since $\cI$ is a covering intervention set, for each vertex $\widehat{i} \in \cV$, there exists an intervention $\widehat{I} \in \cI$ under which all the parents $\pa\left(\widehat{i}\right)$ are intervened upon, but $\widehat{i}$ itself is not. The intervention $\widehat{I}$ has already been performed $T/|\cI|$ times by the algorithm. Using these $T/|\cI|$ independent samples and for a specific assignment $\zd_{\pa\left( \widehat{i} \right)}$ (induced under $\widehat{I}$), we have the estimate $\cPh{\zd_{\widehat{i}}}{\zd_{\pa\left( \widehat{i} \right)}}{}$. Hence, for every vertex $i \in \cV$ and every assignment $\zd_{\pa(i)}$, the algorithm has an estimate $\cPh{\zd_i}{\zd_{\pa(i)}}{}$ in hand. Using these probability estimates, the algorithm computes the reward estimates $\widehat{\mu} (A)$  for each intervention $A \in \cA$; see equation (\ref{eq:calc_emperical}). Finally, enumerating over the given set $\cA$, the algorithm returns the intervention with the maximum estimated reward. It is relevant to note that this patently simple algorithm requires a technically involved regret analysis (detailed in Section \ref{section:regret-analysis}). Indeed, the analysis is a key contribution of the current work. 

\begin{algorithm}[ht!]
	\caption{Covering Interventions Algorithm}
	\label{algo: main algo fully obs setting}
	\noindent
	\textbf{Input:} Causal graph $\cG$, target intervention set  $\cal{A}$, and time horizon $T \in \mathbb{Z}_+$. \\
	\vspace{-10pt}
 	\begin{algorithmic}[1]
		\STATE Set $\cI \leftarrow \textsc{ConstructCover}(\cG)$. \label{line:construct-cover}
        \STATE For each $I \in \cI$, intervene with $\Do(I)$ and collect $\frac{T}{|\cal{I}|}$ samples. \label{line:do-intervene}
		\FOR{ each intervention $A \in \cal{A}$}
		\STATE Compute $\widehat{\mu} (A)$ using equation (\ref{eq:calc_emperical}). \label{line:estimate-reward}
		\ENDFOR
  	\RETURN $\argmax_{A \in \cal{A}} \ \widehat{\mu} (A)$.
	\end{algorithmic}
\end{algorithm}

\begin{theorem}\label{thm: main theorem full obs}
Let $\cG$ be any given causal graph with $N$ vertices and in-degree at most $d$. Also, let $\cI$ be a covering intervention set of $\cG$. Then, Algorithm \ref{algo: main algo fully obs setting}---when executed for any (moderately large) time horizon $T$---achieves simple regret 
\begin{align*}
    R_T = O\left( \sqrt{ \frac{N |\cI| \log{(|\cA|T)} }{T} }\right).
\end{align*} Hence, using Lemma \ref{lem:covers}, we obtain the following bound on the simple regret of Algorithm \ref{algo: main algo fully obs setting} 
	\begin{align*}
		R_T = {O} \left( \sqrt{ \frac{N \ d 2^d \ \log |\cA| }{T} } \log T  \right).
	\end{align*}
\end{theorem}
For graphs with additional structure (e.g. bounded out degree or trees), one can obtain covering intervention sets with size smaller than the one provided in Lemma \ref{lem:covers} (see Lemma 4.2 in \cite{acharya2018learning}). Since the regret guarantee of Algorithm \ref{algo: main algo fully obs setting} depends on the size of the covering intervention set, the simple regret bound improves for such specific graphs.

\subsection{Regret Analysis}
\label{section:regret-analysis}
We first provide a standard concentration bound which will be used in the analysis. 
\begin{lemma} [Hoeffding's Inequality] \label{lem:hoeff}
Let $Z_1, \ldots, Z_n$ be independent bounded random variables with $Z_i \in [a_i, b_i]$, for all $i \in [n]$. Then, for all $\varepsilon \geq 0$:
\begin{align*}
\mathbb{P}\left\{ \left| \sum_{i=1}^n\left(Z_i-\mathbb{E}\left[Z_i\right]\right) \right|  \geq \varepsilon \right\} \leq 2 \exp \left(-\frac{2  \varepsilon ^2}{\sum_{i=1}^n (b_i-a_i)^2}\right).
\end{align*}
\end{lemma}

\vspace*{7pt}

To begin the regret analysis, we note that, for each intervention $A \in \cA$, the estimate $\muh(A)$ can be expressed as 
\begin{align*}
	\widehat{\mu} \left(A \right) & = \sum_{\zd \in Z(A)} \prod_{ i \in \cV(A)} \left(\cP{\zd_i}{\zd_{\pa(i)}}{} +  \Delta \cP{\zd_i}{\zd_{\pa(i)}}{}\right).
 \end{align*}
Expanding the product, we obtain 
\begin{align*}
	\widehat{\mu} \left(A \right)  = \mu (A) + & \sum_{\zd \in Z(A)} \bigg( \sum_{ i \in \cV (A)} \Delta \cP{\zd_i}{\zd_{\pa(i)}}{} \prod_{ j \in \cV(A), j\neq i}  \cP{\zd_j}{\zd_{\pa(j)}}{} \ \ + \ \  \mathcal{L}_\zd \bigg).
\end{align*}
Here, $\mathcal{L}_\zd$ represents all the product entries in the expansion that include more than one error term of the form $\Delta \mathcal{P}(\cdot\mid \cdot)$. Specifically, 
\begin{align}
	\mathcal{L}_\zd  &= \sum_{k = 2}^{|\cV (A)|} \sum_{\substack{U \subseteq \cV(A) \\ |U| = k}} \Bigg[\bigg(\prod_{i \in U } \Delta \cP{\zd_i}{\zd_{\pa(i)}}{} \bigg) \times \nonumber \\
    &\qquad\qquad\quad\qquad\bigg(\prod_{j \in \cV(A) \setminus U} \cP{\zd_j}{\zd_{\pa(j)}}{} \bigg)\Bigg] \label{eqn:Lz}
\end{align}
We further write $\mathcal{H}_\zd$ to represent the sum of the entries with a single error term:
\begin{align}
	\mathcal{H}_\zd \coloneqq   \sum_{ i \in \cV (A)} \Delta \cP{\zd_i}{\zd_{\pa(i)}}{} \prod_{ \substack{j \in \cV(A)\\ j\neq i}}  \cP{\zd_j}{\zd_{\pa(j)}}{}  \label{eqn:Hz}
\end{align}
Hence,
\begin{align*}
    \muh(A) - \mu(A) = \sum_{\zd \in Z(A)} \left( \mathcal{H}_\zd + \mathcal{L}_\zd
    \right).
\end{align*}
We will establish upper bounds on the sums of $\mathcal{L}_\zd$s and $\mathcal{H}_\zd$s and in Lemma \ref{lem:l-bound} and Lemma \ref{lem:h-bound}, respectively. These lemmas show that the sum of the $\mathcal{H}$ terms dominates the sum of the $\mathcal{L}$ terms. Furthermore, these bounds imply that the estimated reward  $\muh(A)$ is sufficiently close to the true expected reward $\mu(A)$ for each $A \in \cA$.  

\begin{restatable}{lemma}{LemmaErrBound}
\label{lem:er_bound}
For estimates obtained via a covering intervention set $\cI$, as in Algorithm \ref{algo: main algo fully obs setting}, write $\mathcal{E}$ to denote the event that $|\Delta \cP{\zd_i}{\zd_{\pa(i)}}{}| \leq \sqrt{\frac{|\cI| (d+ \log{(NT)})}{T}}$, for all vertices $i \in \cV$ and all assignments $\zd_{\pa(i)} \in \{0,1\}^{|\pa(i)|}$. Then, $\mathbb{P}\{ \mathcal{E}\} \geq \left( 1-\frac{2}{T} \right)$.
\end{restatable}
\begin{proof}
Since $\cI$ is a covering intervention set, for each conditional distribution $\cP{\zd_i } {\zd_{\pa(i)}}{}$, we have at least $\frac{T}{|\cI|}$ independent samples. Now, we invoke Lemma \ref{lem:hoeff}, with $\varepsilon = \sqrt{\frac{|\cI| \log{(2^d NT)}}{T}}$, and apply the union bound over all $i \in [N]$ and all assignments to $\pa(i)$. This gives us the desired probability bound. \
\end{proof}

\begin{restatable}{lemma}{LemmaLBound}
\label{lem:l-bound}
For estimates obtained via a covering intervention set $\cI$, as in Algorithm \ref{algo: main algo fully obs setting}, the following event holds with probability at least $\left( 1-\frac{2}{T} \right)$:
\begin{align*}
\sum_{\zd \in Z(A)} \left| \mathcal{L}_\zd  \right| \leq 4(N\eta)^2 \qquad \text{for all $A \in \cA$.}
\end{align*}
Here, parameter $\eta = \sqrt{\frac{|\cI| (d+\log{(NT)})}{T}} $ and $T$ is moderately large.    
\end{restatable}
\begin{proof}
    We will use the fact that each error term in $\mathcal{L}_\zd$ satisfies the bound stated in Lemma \ref{lem:er_bound}. Moreover, we utilize the graph structure to marginalize variables that do not appear in an expansion of $\mathcal{L}_\zd$.
	\begin{align*}
		\sum_{\zd \in Z(A)} |\mathcal{L}_\zd | \leq \sum_{\zd \in Z (A)} \sum_{k = 2}^{|\cV (A)|} \sum_{\substack{U \subseteq \cV(A) \\ |U| = k \ \ }} \left(\prod_{i \in U }  \left| \Delta \cP{\zd_i}{\zd_{\pa(i)}}{}\right| \right) \left(\prod_{j \in \cV(A) \setminus U} \cP{\zd_j}{\zd_{\pa(j)}}{} \right)\\
		= \sum_{k = 2}^{|\cV (A)|} \sum_{\zd \in Z (A)}  \sum_{\substack{U \subseteq \cV(A) \\ |U| = k \ \ }} \left(\prod_{i \in U } \left |\Delta \cP{\zd_i}{\zd_{\pa(i)}}{} \right |  \right) \left(\prod_{j \in \cV(A) \setminus U} \cP{\zd_j}{\zd_{\pa(j)}}{} \right).
	\end{align*}

	First, we upper bound each term considered in the outer-most sum. Towards this, let $U =\{V_{x_1}, V_{x_2}, \ldots, V_{x_k}\}$ to be a subset of vertices that appears in the inner sum. Here, $x_1 < x_2 < \ldots < x_k$ and, as mentioned previously, the indexing of the vertices respects a topological ordering over the causal graph. In the derivation below, we will split the sum into $k$ parts, $\sum_{\zd_{[1:x_1]}} \sum_{\zd_{(x_1:x_2]}} \ldots \sum_{\zd_{(x_k:N]}} $, and individually bound the marginalized probability distribution.
	\begin{align}
		&\sum_{\zd \in Z (A)}  \sum_{\substack{U \subseteq \cV(A) \\ |U| = k}} \left(\prod_{i \in U } \left |\Delta \cP{\zd_i}{\zd_{\pa(i)}}{} \right |  \right) \left(\prod_{j \in \cV(A) \setminus U} \cP{\zd_j}{\zd_{\pa(j)}}{} \right) \nonumber\\ 
		&\leq \sum_{\substack{U \subseteq \cV(A) \\ |U| = k}} \sum_{\zd \in Z (A)} \eta^k   \left(\prod_{j \in \cV(A) \setminus U} \cP{\zd_j}{\zd_{\pa(j)}}{} \right) \nonumber \tag{via Lemma \ref{lem:er_bound}, $\left|\Delta \cP{\zd_i}{\zd_{\pa(i)}}{} \right| \leq \eta$ }\\ 
		&= \sum_{\substack{U \subseteq \cV(A) \\ |U| = k}} \eta ^k 
		\sum_{\zd_{[1: \x_1]} \in Z_{[1: \x_1]}(A)} 
		\left(\prod_{j_1 \in  \cV_{[1:x_1) }(A) } \cP{\zd_{j_1}}{\zd_{\pa(j_1)}}{} \right) \sum_{\zd_{(x_1: \x_2]} \in Z_{(x_1: x_2]}(A)} \left(\prod_{j_2 \in  \cV_{(x_1:x_2) }(A) } \cP{\zd_{j_2}}{\zd_{\pa(j_2)}}{} \right)  \nonumber
		\ldots
		 \\& \sum_{\zd\in Z_{(x_i: x_{i+1}]}(A)} \left(\prod_{j_i \in \cV_{(x_i:x_{i+1}) }(A) } \cP{\zd_{j_i}}{\zd_{\pa(j_i)}}{} \right) \ldots \sum_{\zd_{(x_k : N ]} \in Z_{(x_k: N]}(A)} \left(\prod_{j_k \in \cV_{(x_k:N] }(A) } \cP{\zd_{j_k}}{\zd_{\pa(j_k)}}{} \right) \label{eqn:sec}
	\end{align}

	The last term in the above expression can be bounded as follows 
	\begin{align*}
		\sum_{\zd_{(x_k: N]} \in Z_{(x_k: N]}(A)} \left(\prod_{j \in \cV_{(x_k:N] }(A) } \cP{\zd_j}{\zd_{\pa(j)}}{} \right) &= \sum_{\zd_{(x_k: N]} \in Z_{(x_k: N]}(A)} \prob_{\Do(A)} \left[ \cV_{(x_k: N]}(A) = \zd_{(x_k: N]} | \pa (\cV_{(x_k: N]}(A)) \right]\\
  &=\prob_{\Do(A)} \left[ V_N = 1 | \pa \left( \cV_{(x_k:N] }(A) \right)\right] \leq 1.
	\end{align*}
	For all other terms, we have the following inequality 
	\begin{align*}
		\sum_{\zd\in Z_{(x_i: x_{i+1}]}(A)}& \left(\prod_{j \in \cV_{(x_i:x_{i+1}) }(A) } \cP{\zd_j}{\zd_{\pa(j)}}{} \right) 
		\\ &= \sum_{\zd_{x_{i+1}} \in \{0,1\}} \ \sum_{\zd_{(x_i: x_{i+1})} \in Z_{(x_i: x_{i+1})}(A)} \left(\prod_{j \in \cV_{(x_i:x_{i+1}) }(A) } \cP{\zd_j}{\zd_{\pa(j)}}{} \right)
		\\ &= \sum_{\zd_{x_{i+1}} \in \{0,1\}} \sum_{\zd_{(x_i: x_{i+1})} \in Z_{(x_i: x_{i+1})}(A)} \prob_{\Do(A)} \left[ \cV_{(x_i:x_{i+1}) }(A) = \zd_{(x_i: x_{i+1})} | \pa \left( \cV_{(x_i:x_{i+1}) }(A) \right)\right]
  \\ & \leq \sum_{\zd_{x_{i+1}} \in \{0,1\}} 1  \\&=  2.
	\end{align*}

Substituting in (\ref{eqn:sec}), we get 
\begin{align*}
	\sum_{\zd \in Z (A)}  \sum_{\substack{U \subseteq \cV(A) \\ |U| = k}} \left(\prod_{i \in U } \left |\Delta \cP{\zd_i}{\zd_{\pa(i)}}{} \right |  \right) \left(\prod_{j \in \cV(A) \setminus U} \cP{\zd_j}{\zd_{\pa(j)}}{} \right) \leq \sum_{\substack{U \subseteq \cV(A) \\ |U| = k}} \left( 2 \eta\right) ^k = \binom{N}{k} \left( 2 \eta\right) ^k.
\end{align*}

Therefore, the sum $\sum_{\zd \in Z(A)} \left| \mathcal{L}_\zd  \right|$ satisfies
\begin{align*}
	\sum_{\zd \in Z(A)} |\mathcal{L}_\zd | & \leq \sum_{k=2}^{N}\binom{N}{k} \left( 2 \eta\right) ^k \\
	&= \sum_{k=0}^{N}\binom{N}{k} \left( 2 \eta\right) ^k  - 2N\eta - 1 \\
	&=(1+2\eta)^N-2N\eta-1\\
	&\leq e^{2N\eta}-2N\eta -1\\
	&\leq 1 + 2N\eta + (2N\eta)^2 - 2N \eta -1 \tag{with $\eta \leq \frac{1}{2N}$}\\
	&\leq 4N^2\eta^2.
\end{align*}
The lemma stands proved. 
\end{proof}

\begin{restatable}{lemma}{LemmaHBound}\label{lem:h-bound}
For estimates obtained via a covering intervention set $\cI$, as in Algorithm \ref{algo: main algo fully obs setting}, the following event holds with probability at least $\left( 1-\frac{2}{T} \right)$:
    \begin{align*}
        \left| \sum_{\zd \in Z(A)}\mathcal{H}_\zd \right| \leq \sqrt{\frac{N |\cI| \log{(|\cA| T)}}{T}} \qquad \text{for all $A \in \cA$.}
    \end{align*}
\end{restatable}
\begin{proof}
The definition of $\mathcal{H}_\zd$ gives us
{
	\begin{align*} 
		&\left| \sum_{\zd \in Z(A)}\mathcal{H}_\zd \right|\\
            &= \left|  \sum_{\zd \in Z(A)}\sum_{ i \in \cV (A)} \Delta \cP{\zd_i}{\zd_{\pa(i)}}{} \prod_{ j \in \cV(A), j\neq i}  \cP{\zd_j}{\zd_{\pa(j)}}{} \right| \\
		&=  \left|  \sum_{ i \in \cV (A)} \sum_{\zd \in Z(A)} \Delta \cP{\zd_i}{\zd_{\pa(i)}}{} \prod_{ j \in \cV(A), j\neq i}  \cP{\zd_j}{\zd_{\pa(j)}}{} \right| \\
		&= \left|  \sum_{ i \in \cV (A)} \sum_{\substack{\zd_{[1:i]} \in \\  Z_{[1:i]}(A)}} \Delta \cP{\zd_i}{\zd_{\pa(i)}}{} \prod_{ j \in \cV_{[1:i)}(A)} \cP{\zd_j}{\zd_{\pa(j)}}{} \sum_{\substack{\zd_{(i: N ]} \in \\ Z_{(i:N]}(A)}} \prod_{ k \in \cV_{(i:N]}(A)}  \cP{\zd_k}{\zd_{\pa(k)}}{} \right| \\
		&= \left| \sum_{i \in \cV (A) } \sum_{\substack{\zd_{[1:i]} \in \\  Z_{[1:i]}(A)}} \Delta \cP{\zd_i}{\zd_{\pa(i)}}{} \prod_{ j \in \cV_{[1:i)}(A)} \cP{\zd_j}{\zd_{\pa(j)}}{} \sum_{\substack{\zd_{(i: N ]} \in \\ Z_{(i:N]}(A)}} \prob_{\Do(A)} \left[\cV_{(i:N]}(A) = \zd_{(i:N ]} \mid \pa\left(\cV_{(i:N]}(A)\right)  \right] \right|\\
		&= \left| \sum_{ i \in \cV (A) } \sum_{\substack{\zd_{[1:i]} \in \\  Z_{[1:i]}(A)}} \Delta \cP{\zd_i}{\zd_{\pa(i)}}{} \prod_{ j \in \cV_{[1:i)}(A)} \cP{\zd_j}{\zd_{\pa(j)}}{} \prob_{\Do(A)} \left[V_N = 1 \mid \pa\left(\cV_{(i:N]}(A)\right)  \right] \right|\\    
		&= \left| \sum_{ i \in \cV (A) } \sum_{\zd_i \in \{0,1\}}\sum_{\substack{\zd_{[1:i)} \in \\ Z_{[1:i)}(A)}} \Delta \cP{\zd_i}{\zd_{\pa(i)}}{} \prod_{ j \in \cV_{[1:i)}(A)} \cP{\zd_j}{\zd_{\pa(j)}}{}
		\prob_{\Do(A)} \left[V_N = 1 \mid \pa\left(\cV_{(i:N]}(A)\right)  \right] \right|\\
  &= \left| \sum_{ i \in \cV (A) } \sum_{\zd_i \in \{0,1\}}\sum_{\substack{\zd_{\pa(i)} \in \\ Z_{\pa(i)}(A)}} \Delta \cP{\zd_i}{\zd_{\pa(i)}}{} \sum_{\substack{\zd_{\ac(i)} \in \\ Z_{\ac(i)}(A)}} \prod_{ j \in \cV_{[1:i)}(A)} \cP{\zd_j}{\zd_{\pa(j)}}{} \prob_{\Do(A)} \left[V_N = 1 \mid \pa\left(\cV_{(i:N]}(A)\right)  \right] \right|.
\end{align*}
}
 Recall that $\ac(i) = [1,i) \setminus \pa(i) $ and write 
 \begin{align}
     c_i(z_i, \zd_{\pa (i)}) \coloneqq  \sum_{\substack{\zd_{\ac(i)} \in \\ Z_{\ac(i)}(A)}} \ \prod_{ j \in \cV_{[1:i)}(A)} \cP{\zd_j}{\zd_{\pa(j)}}{} \prob_{\Do(A)} \left[V_N = 1 \mid \pa\left(\cV_{(i:N]}(A)\right)  \right] \label{eqn:def-c}
 \end{align}
Also, as a shorthand for $z_i=1$ and $z_i=0$ we will write $1_i$ and $0_i$, respectively. With these notations, we have 
\begin{align*}
\left| \sum_{\zd \in Z(A)}\mathcal{H}_\zd \right| 
  &= \left| \sum_{ i \in \cV (A) } \sum_{\zd_i \in \{0,1\}} \sum_{\zd_{\pa (i)} \in Z_{\pa(i)}(A) } \Delta \cP{\zd_i}{\zd_{\pa(i)}}{} c_i(\zd_i,\zd_{\pa(i)}) \right| \\
		&=\left| \sum_{ i \in \cV (A) } \sum_{\zd_{\pa (i)} \in Z_{\pa(i)}(A) } \Delta \cP{1_i}{\zd_{\pa(i)}}{} \left(c_i(1_i,\zd_{\pa(i)}) - c_i(0_i,\zd_{\pa(i)}) \right)\right| \tag{since  $\Delta \cP{1_i}{\zd_{\pa(i)}}{} = -  \Delta \cP{0_i}{\zd_{\pa(i)}}{}$}
	\end{align*}

Since $\cI$ is a covering intervention set, for each pair $(i,\zd_{\pa(i)})$, there exists an intervention $I \in \cI$ such that intervening $\Do(I)$ provides a sample from the conditional probability distribution $\prob [ V_i =1 \mid \pa(V_i) = \zd_i]$. Hence, Line \ref{line:do-intervene} of the algorithm provides at least $\frac{T}{|\cI|}$ independent samples from the conditional distribution $\prob [ V_i =1 \mid \pa(V_i) = \zd_i]$. We write the $s^{th}$ sample for this conditional distribution by $Y_s(i,\zd_{\pa(i)})$. Now, we have 
\begin{align*}
	\left| \sum_{\zd \in Z(A)}\mathcal{H}_\zd \right|  = \left| \sum_{ i \in \cV (A) } \sum_{\substack{\zd_{\pa (i)} \in \\ Z_{\pa(i)}(A)}} \frac{|\cI|}{T} \left( \sum_{s=1}^{T/|\cI|} Y_s(i,\zd_{\pa(i)}) - \cP{1_i}{\zd_{\pa(i)}}{} \right)   (c_i(1_i,\zd_{\pa(i)}) - c_i(0_i,\zd_{\pa(i)}))\right|.
\end{align*}
We will apply Hoeffding's inequality (Lemma \ref{lem:hoeff}) to bound the above expression. Note that in this expression, besides $Y_s(i,\zd_{\pa(i)})$-s, all the other terms are deterministic. In particular, we show in Claim \ref{lem:bound_c} (stated and proved below) that $\sum_{\zd_{\pa (i)} \in Z_{\pa(i)}(A)} (c(1_i,\zd_{\pa(i)}) - c(0_i,\zd_{\pa(i)}))^2  \leq 1$, for all $i$. Hence, for any $A \in \cA$, Lemma \ref{lem:hoeff} gives us 
\begin{align*}
	\prob \left( \left| \sum_{\zd \in Z(A)}\mathcal{H}_\zd \right| \geq \varepsilon \right) &\leq 2 \ \mathrm{exp}\left( \frac{ - T \varepsilon^2}{|\cI| \sum_{ i \in \cV (A) } \sum_{\zd_{\pa (i)} \in Z_{\pa(i)}(A) } (c_i(1_i,\zd_{\pa(i)}) - c_i(0_i,\zd_{\pa(i)}))^2} \right) \\
		&\leq 2 \ \mathrm{exp} \left( \frac{ - T \varepsilon^2}{ |\cI| \ |\cV(A)| } \right) \tag{via Claim \ref{lem:bound_c}}\\ 
		&\leq 2 \ \mathrm{exp} \left( \frac{ - T \varepsilon^2}{ |\cI| \ N } \right).
	\end{align*}
	Setting $\varepsilon = \sqrt{\frac{N \ |\cI| \log{(|\cA|  T)}}{T}}$ and taking union bound over all $A \in \cA$, gives us the required probability bound. This completes the proof of the lemma. 
 \end{proof}
We next establish the claim used in the proof of Lemma \ref{lem:h-bound}.
\begin{claim} \label{lem:bound_c}
	\begin{align*}
		\sum_{\zd_{\pa (i)} \in Z_{\pa(i)}(A)} (c(1_i,\zd_{\pa(i)}) - c(0_i,\zd_{\pa(i)}))^2  \leq 1.
	\end{align*}
\end{claim}
\begin{proof}
The definition of $c(\zd_i, \zd_{\pa(i)})$ (see equation (\ref{eqn:def-c})) gives us 
{ \allowdisplaybreaks
	\begin{align*}
		&| c(1_i,\zd_{\pa(i)}) - c(0_i,\zd_{\pa(i)}) | 
		\\&= \biggl| \sum_{\zd \in Z_{\ac(i)}(A) } \prod_{ j \in \cV_{[1:i)}(A)} \cP{\zd_j}{\zd_{\pa(j)}}{} \prob_{\Do(A)} \left[\cV_{[i+1:N]}(A) = \zd_{[i+1:N ]} \mid \pa\left(\cV_{[1: i]}(A) \right) = (\zd_{[1:i)} \cup 1_{i} )  \right] - 
		\\ & \qquad \qquad \sum_{\zd \in Z_{\ac(i)}(A) } \prod_{ j \in \cV_{[1:i)}(A)} \cP{\zd_j}{\zd_{\pa(j)}}{} \prob_{\Do(A)} \left[\cV_{[i+1:N]}(A) = \zd_{[i+1:N ]} \mid \pa\left(\cV_{[1: i]}(A) \right)= (\zd_{[1:i)} \cup 0_{i} )  \right] \biggr|\\
        &=  \biggl| \sum_{\zd \in Z_{\ac(i)}(A) } \prod_{ j \in \cV_{[1:i)}(A)} \cP{\zd_j}{\zd_{\pa(j)}}{}  \biggl( \prob_{\Do(A)} \left[\cV_{[i+1:N]}(A) = \zd_{[i+1:N ]} \mid \pa (\cV_{[1: i]}(A)) = (\zd_{[1:i)} \cup 1_{i} ) \ \right] - \\ & \qquad \prob_{\Do(A)} \left[\cV_{[i+1:N]}(A) = \zd_{[i+1:N ]} \mid \pa(\cV_{[1: i]}(A) )= (\zd_{[1:i)} \cup 0_{i} )  \right] \biggr)  \biggr| \\
		&\leq   \sum_{\zd \in Z_{\ac(i)}(A) } \prod_{ j \in \cV_{[1:i)}(A)} \cP{\zd_j}{\zd_{\pa(j)}}{} \biggl| \prob_{\Do(A)} \left[\cV_{[i+1:N]}(A) = \zd_{[i+1:N ]} \mid \pa (\cV_{[1: i]}(A)) = (\zd_{[1:i)} \cup 1_{i})\right] - \\ & \qquad \qquad \prob_{\Do(A)} \left[\cV_{[i+1:N]}(A) = \zd_{[i+1:N ]} \mid \pa (\cV_{[1: i]}(A)) = (\zd_{[1:i)} \cup 0_{i} )  \right] \biggr| \\
		& \leq \sum_{\zd \in Z_{\ac(i)}(A) } \prod_{ j \in \cV_{[1:i)}(A)} \cP{\zd_j}{\zd_{\pa(j)}}{} \\
		& = \prob_{\Do(A)}\left[ \cV_{\pa(i)}(A) = \zd_{\pa(i)}\right].
	\end{align*}
}
Hence, under intervention $A \in \cA$, we have 
	\begin{align*}
		\sum_{\zd_{\pa (i)} \in Z_{\pa(i)}(A)} (c(1_i,\zd_{\pa(i)}) - c(0_i,\zd_{\pa(i)}))^2  &\leq \sum_{\zd_{\pa (i)} \in Z_{\pa(i)}(A)} |c(1_i,\zd_{\pa(i)}) - c(0_i,\zd_{\pa(i)})|\\
		&\leq \sum_{\zd_{\pa (i)} \in Z_{\pa(i)}(A)} \prob_{\Do(A)}\left[ \cV_{\pa(i)}(A) = \zd_{\pa(i)}\right] \\
		&\leq 1.
	\end{align*}
This completes the proof of the claim.
\end{proof}

Recall that the random variables $\mathcal{L}_\zd$ and $\mathcal{H}_\zd$ depend on the error terms $\Delta \cP{\zd_i}{\zd_{\pa(i)}}{}$. Moreover, in Lemma \ref{lem:l-bound} and \ref{lem:h-bound}, the considered sums can range over exponentially many such variables. The technically involved contribution of these lemmas is that we obtain small error bounds even in such settings of exponentially large sums.

\subsection{Proof of Theorem \ref{thm: main theorem full obs}}
Lemma \ref{lem:covers} implies that, with probability at least $\left(1 - \frac{1}{T}\right)$, the set $\cI$ obtained in Line \ref{line:construct-cover} of Algorithm \ref{algo: main algo fully obs setting} is indeed a covering intervention set. We combine this guarantee with Lemmas \ref{lem:l-bound} and \ref{lem:h-bound}. In particular, with probability at least $\left(1- \frac{5}{T}\right)$, we have, for all $A \in \cA$:
\begin{align*}
  & \left|\mu(A)  - \muh(A)\right|  \\ & \ \  \ = \left|\sum_{\zd \in Z(A)} \left( \mathcal{H}_\zd + \mathcal{L}_\zd \right) \right|\\
        &  \ \ \ \leq \sqrt{\frac{N |\cI| \log{(|\cA| T)}}{T}} + \frac{4 N^2 |\cI| (d+\log{(NT)})}{T}\\
        & \ \ \ \leq 2\sqrt{\frac{N |\cI| \log{(|\cA| T)}}{T}} \tag{for $T \gtrsim  N^3$}
\end{align*}
Let $A_T \in \cA$ be the intervention returned by Algorithm \ref{algo: main algo fully obs setting} (after $T$ rounds of interventions), i.e., $A_T = \argmax_{A\in \cA} \ \muh(A)$. In addition, $A^* = \argmax_{A \in \cA} \ \mu(A)$ be the optimal intervention. Hence, with probability at least $\left(1-\frac{5}{T} \right)$, we have 
    \begin{align}
        \mu(A^*)- \mu(A_T) \leq 4  \sqrt{\frac{N |\cI| \log{(|\cA| T)}}{T}} 
    \end{align}
This guarantee gives us the desired upper bound on the simple regret, $R_T$, of Algorithm \ref{algo: main algo fully obs setting}: 
\begin{align*}
    R_T &= \E \left[ \mu(A^*)- \mu(A_T) \right]\nonumber\\ &\leq  \left(4  \sqrt{\frac{N |\cI| \log{(|\cA| T)}}{T}}\right) \left(1-\frac{5}{T}\right) + \frac{5}{T}\nonumber \\ 
    &\leq \ 5 \sqrt{\frac{N |\cI| \log{(|\cA| T)}}{T}}.
\end{align*}
Since the size of the covering intervention set satisfies $|\cI| = 3d \cdot  2^d (\log N +2d + \log T )$ (see Lemma \ref{lem:covers}), we also have the following explicit form of the simple regret bound
\begin{align*}
	R_T = {O} \left( \sqrt{ \frac{N \ d 2^d \ \log |\cA| }{T} } \log T  \right).
\end{align*}
The theorem stands proved.

\section{Algorithm for Graphs with Unobserved Variables}
\label{section: Algorithm for SMBN}
\label{section:smbn}

We now extend our algorithm to causal graphs with unobserved variables. In particular, we study Semi Markovian Bayesian Networks (SMBNs) where we have the causal graph defined as $\cG = (\cV, E, E')$. Here, $E$ is the set of directed edges, and $E'$ is the set of bi-directed edges denoting the presence of an unobserved common parent. Any general causal graph can be projected to an equivalent SMBN \cite{tian2002testable}. Hence,  without loss of generality and throughout this section, we assume that the causal graph is an SMBN. It is relevant to note that in an SMBN all the vertices in $\cV$ are observable and the unobserved variables are encapsulated by the edges $E'$. 

Assume that the vertices $\cV$ are topologically ordered (based on the directed edges $E$) and the ordering is preserved in any subset  $\cU \subset \cV$. 
The SMBN graph $\cG$ can be decomposed into a disjoint set of vertices known as \emph{confounded components} (c-components), where  each c-component is the maximal set of vertices that are connected through a bi-directed edge in $E'$. Let $\cC(A)$ denote all the c-components of $\cG$ under intervention $A$. We use  $C_i$ to denote the $i^{th}$ c-component in $\cC(A)$. We assume that any $C_i$ maintains the topological order (induced by the directed edges $E$). Now, the joint distribution of the vertices for an assignment $\zd \in Z(A)$, under intervention $A$, can be written as
\begin{equation*}
   \prob\left[ V = \zd \mid \Do(A)\right] = \prod_{C_i \in \cC(A)} \cP{\zd_{C_i}}{}{\zd_{\pa(C_i)}}.
\end{equation*}


Under an empirical estimation, we represent the $s^{th}$ sample from the distribution $\cP{\zd_{C_i}}{}{\zd_{\pa(C_i)}}$ via the indicator random variable $Y_{s}(\zd_{C_i}, \zd_{\pa(C_i)})$, which takes the value one when $ \cV_{C_i} = \zd_{C_i} $, else it takes the value zero. Let $n\left(C_i, \zd_{\pa(C_i)} \right)$ be the total number of samples in this for the pair $(C_i, \zd_{\pa(C_i)})$. We compute the probability estimates as follows 
\begin{align}
    \cPh{\zd_{C_i}}{}{\zd_{\pa(C_i)}} & = \frac{\sum_{s=1}^{T_i} Y_{s}\left(\zd_{C_i}, \zd_{\pa(C_i)}\right) }{n\left(C_i, \zd_{\pa(C_i)} \right) } \label{eq:calc_pr_smbn} \\ 
    \widehat{\mu}(A) & =  \sum_{\zd \in Z(A)} \prod_{ C_i \in \cC(A)} \cPh{\zd_{C_i}}{}{\zd_{\pa(C_i)}} \label{eq:calc_emperical_smbn} 
\end{align}
Next, we extend the definition of covering intervention set (Definition \ref{definition: CIS}) for SMBNs:
\begin{definition}\label{defn:CIS-SMBN}
A set of intervention $\cI$ is a covering intervention set if for all subsets $S$ of every c-component in $\cG$, and every assignment $\zd_{\pa(S)} \in \{0,1\}^{|\pa(S)|}$ there exists and $I \in \cI$ with the properties that 
\begin{itemize}
    \item No vertex in $S$ is intervened in $I$.
    \item Every vertex in $\pa (S)$ is intervened in $I$.
    \item $\pa(S)$ is intervened with assignment $\zd_{\pa(S)}$. 
\end{itemize}
\end{definition}

We construct a covering intervention set as before using the randomized method in \cite{acharya2018learning}. The next lemma states that the randomized method provides a covering intervention set of size $\widetilde{O}(\log N)$ even in the case of SMBNs. This result is a direct implication of Lemma 4.2 in \cite{acharya2018learning}.
\begin{lemma} [\cite{acharya2018learning}] \label{lem:covers SMBN}
    For any moderately large $T \in \mathbb{Z}_+$ and any causal graph $\cG$---with in-degree at most $d$ and c-components of size at most $\ell$---there exists a covering intervention set $\cI$ of size $k = (3d)^\ell \  2^{\ell d} (\log N +2\ell d + \log T )$. Furthermore, such a set $\cI$ can be found with probability at least $\left(1 - \frac{1}{T}\right)$.
\end{lemma}

The simple regret algorithm for SMBNs is exactly the same as Algorithm \ref{algo: main algo fully obs setting}, except for the following two changes:
\begin{itemize}
    \item The \textsc{ConstructCover} subroutine returns a covering intervention set of size  $(3d)^\ell 2^{\ell d} (\log N +2\ell d + \log T )$. 
    \item We use equation (\ref{eq:calc_emperical_smbn}) to compute the estimates $\muh(A)$ for each $A \in \cA$.
\end{itemize}

The theorem below is the main result of this section. 

\begin{restatable}{theorem}{TheoremSMBN}
\label{thm: main theorem SMBN}
	Let $\cG$ be any given causal graph over $N$ vertices and with c-components of size at most $\ell$. Also, let the in-degree of the vertices in $\cG$ be at most $d$. Then, for any (moderately large) time horizon $T$ and given any covering intervention set $\cI$ of $\cG$, Algorithm \ref{algo: main algo fully obs setting}  achieves simple regret 
 \begin{align*}
     \Rg_T = O \left( \sqrt{ \frac{N \ 2^d \ 4^\ell  \ |\cI|  \log{(|\cA | T)} }{T} } \right).
 \end{align*} 
 Hence, using Lemma \ref{lem:covers SMBN}, we obtain the following bound on the simple regret 
	\begin{align*}
		\Rg_T = O \left( \sqrt{ \frac{N \  (3  d \ 8^d )^\ell  \  \log{|\cA | } }{T}  }\log{T} \right).
	\end{align*}
\end{restatable}

\subsection{Regret Analysis for SMBNs}
\label{section:regret-analysis SMBN}

 We introduce the notion of \emph{pseudo parents} of a vertex in a Semi Markov Bayesian Networks (SMBN) graph $\cG$, which we will use throughout the proof. Recall that $\cV$ denotes the set of vertices, and they conform to a topological ordering. We assume that each c-component $C_i$ maintains the ordering. For an intervention $A$, consider any $c$-component $C \in \cC(A)$ with vertices $(U_1, U_2, \ldots, U_m)$, the pseudo parents of a vertex $U_j$ is defined as

\begin{align}
    \qa(j) \coloneqq \pa(\{U_1, U_2, \ldots U_{j} \}) \cup \left\{U_1, U_2, \ldots U_{j-1} \right\}  \label{definition: pseudo_pa}
\end{align}

For any SMBN graph with in-degree at most $d$ and c-components of size at most $\ell$, the size $|\qa(j)|$ is at most $d \ell + \ell$. Furthermore, note that the set $\qa(j)$ will always precede the vertex $V_j$ in any topological ordering of the graph.

The next lemma shows that the distribution of any c-component conditioned on its parents, $\cP{\zd_{C}}{}{\zd_{\pa(C)}}$, can be  factorized into the distribution of individual vertices conditioned on its pseudo parents. This allows us to extend the techniques used for the regret analysis of fully observable graphs  (Section \ref{section:regret-analysis}) to the case of SMBNs. Intuitively, one can view the factorization of an SMBN (under an intervention $A$) as a factorization over a fully observable graph where each vertex $V_j$ has the set $\qa(j)$ as its parents.   

\begin{restatable}{lemma}{LemmaPseudoPa} \label{lem:pseudo_pa}
    For any intervention $A$ and any c-component $C \in \cC(A)$, consisting of vertices $\{U_1, U_2 \ldots U_m \}$, we have  
    \begin{equation*}
        \cP{\zd_{C}}{}{\zd_{\pa(C)}} = \prod_{j \in C} \cP{\zd_j}{\zd_{\qa(j)}}{A}. 
    \end{equation*}
    Here, $\qa(j)$ denotes the set of pseudo parents as defined in equation (\ref{definition: pseudo_pa}).
\end{restatable}
A proof of Lemma \ref{lem:pseudo_pa} appears in Appendix \ref{appendix:regret-analysis SMBN}.

Now, recall that the estimate $\muh(A)$ can be written as 
\begin{alignat*}{2}
	\widehat{\mu} \left(A \right) &=\sum_{\zd \in Z(A)} \prod_{ C_i \in \cC(A)}  \cPh{\zd_{C_i}}{}{\zd_{\pa(C_i)} }  \\
 &= \sum_{\zd \in Z(A)}  \prod_{ i \in \cC(A)} \left( \cP{\zd_{C_i}}{\zd_{\pa(C_i)}}{A} +  \Delta \cP{\zd_{C_i}}{\zd_{\pa(C_i)}}{A}\right)                                                       \\
     &= \mu (A) + \sum_{\zd \in Z(A)} \Biggl( \sum_{ C_i \in \cC (A)}  \Delta \cP{\zd_{C_i}}{\zd_{\pa(C_i)}}{A} \prod_{ C_j \in \cC(A), j\neq i}  \cP{\zd_{C_j}}{\zd_{\pa(C_j)}}{A}+  \\
    &\sum_{\substack{U \subseteq \cC(A) \\ |U| = 2}} \left(\prod_{C_i \in U } \Delta \cP{\zd_{C_i}}{\zd_{\pa(C_i)}}{A} \right) \left(\prod_{C_j \in \cC(A) \setminus U} \cP{\zd_{C_j}}{\zd_{\pa(C_j)}}{A}\right) + \\
    &\sum_{\substack{U \subseteq \cC(A) \\ |U| =3 }} \left(\prod_{C_i \in U } \Delta \cP{\zd_{C_i}}{\zd_{\pa(C_i)}}{A} \right) \left(\prod_{C_j \in \cC(A) \setminus U} \cP{\zd_{C_j}}{\zd_{\pa(C_j)}}{A} \right)    + \cdots \Biggr) \tag{expanding product terms} 
    {}
\end{alignat*}
Here, $\Delta \cP{}{}{} $ denotes the error in the estimate of the conditional probabilities. Let  $\cL$ represent all the product entries in the expansion that include more than one error term ( $\Delta \cP{}{}{}$). Specifically,
\begin{align*}
	\cL & = \sum_{k = 2}^{|\cC (A)|} \sum_{\substack{U \subseteq \cC(A) \\ |U| = k}} \left(\prod_{C_i \in U } \Delta \cP{\zd_{C_i}}{\zd_{\pa(C_i)}}{A} \right) \left(\prod_{C_j \in \cC(A) \setminus U} \cP{\zd_{C_j}}{\zd_{\pa(C_j)}}{A} \right)\\
  &= \sum_{k = 2}^{|\cC (A)|} \sum_{\substack{U \subseteq \cC(A) \\ |U| = k}} \left(\prod_{C_i \in U } \Delta \cP{\zd_{C_i}}{\zd_{\pa(C_i)}}{A} \right) \left( \prod_{\substack{C \in \cC(A) \setminus C_i, \\ j \in C }} \cP{\zd_j}{\zd_{\qa(j)}}{A} \right) \tag{via Lemma \ref{lem:pseudo_pa}}
\end{align*}
We further represent all the entries with a single $\Delta \cP{}{}{}$ term as
\begin{align}
	\cH =  \sum_{ C_i \in \cC (A)} \Delta \cP{\zd_{C_i}}{\zd_{\pa(C_i)}}{A} \prod_{ \substack{C_k \in \cC(A)\\ k\neq i}}  \cP{\zd_{C_k}}{\zd_{\pa(C_k)}}{A} \nonumber \\
        =\sum_{ C_i \in \cC (A)} \Delta \cP{\zd_{C_i}}{\zd_{\pa(C_i)}}{A} \prod_{ \substack{j\in \cV(A) \setminus C_i}}  \cP{\zd_{j}}{\zd_{\qa(j)}}{A} \label{eqn:SMBN-exp-Hz}
\end{align}
Here, the last equality follows from Lemma \ref{lem:pseudo_pa}.
Hence, we have 
\begin{equation}
    \muh(A) - \mu(A) = \sum_{\zd \in Z(A)} \left( \cH + \cL \right) \label{eq:error}
\end{equation}
We will establish upper bounds on the sums of $\mathcal{L}_\zd$s and $\mathcal{H}_\zd$s in Lemma \ref{lem:l-bound-smbn} and Lemma \ref{lem:h-bound-smbn}, respectively. These lemmas show that the sum of the $\mathcal{H}$ terms dominates the sum of $\mathcal{L}$ terms. Furthermore, these bounds imply that the estimated reward  $\muh(A)$ is sufficiently close to the true expected reward $\mu(A)$ for each intervention $A \in \cA$.
\begin{lemma} \label{lem:er_bound SMBN}
For estimates obtained via a covering intervention set $\cI$, as in Algorithm \ref{algo: main algo fully obs setting}, write $\mathcal{E}$ to denote the event that $|\Delta \cP{\zd_{C_i}}{}{\zd_{\pa(C_i)}}| \leq \sqrt{\frac{|\cI|( \ell d+\ell+\log({NT}))}{T}} $ for all c-components $C_i \in \cC(A)$ and for all $A \in \cA$. Then, $\Pr\left\{ \mathcal{E} \right\} \geq \left( 1-\frac{2}{T} \right)$. 
\end{lemma}

\begin{proof}
    Since $\cI$ is a covering intervention set (see Defintion \ref{defn:CIS-SMBN}), for each distribution $\cP{\zd_{C_i }}{}{\zd_{\pa(i)}}$, we have at least $\frac{T}{|\cI|}$ independent samples. Also, note that the total number of distributions to be estimated is at most $2^{(\ell d+ \ell)}N$. This follows from the fact that each c-component---under any intervention---is a subset of a c-component in the original graph $\cG$, and the number of c-components in $\cG$ is at most $N$. Hence, the number of possible distinct c-components (across all intervention) is at most $N2^\ell$. Furthermore, each c-component can have at most $\ell d$ parents with at most $2^{\ell d}$ distinct binary assignments to the parents.

    With this count in hand, we invoke Lemma \ref{lem:hoeff}, with $\varepsilon = \sqrt{\frac{|\cI| ( \log{(2^{\ell d+ \ell}NT)})}{T}} $ and apply the union bound over all $(\zd_{C_i} , \zd_{\pa(C_i)})$ pairs. This gives us the desired probability bound and completes the proof of the lemma. 
\end{proof}

\begin{lemma}\label{lem:l-bound-smbn}
	For estimates obtained via a covering intervention set $\cI$, the following event holds with  probability at least $(1 - \frac{2}{T})$: $$\sum_{\zd \in Z(A)} \left| \cL \right| \leq 4^\ell(N\eta)^2  \qquad \text{for all $A \in \cA$.} $$ Here, parameter $\eta = \sqrt{\frac{ |\cI| ( \ell d + \ell + \log({NT}))}{T}}$ and $T$ is moderately large.
\end{lemma}
\begin{proof}
    We use the fact that each error term in $\cL$ satisfies the bound stated in Lemma \ref{lem:er_bound SMBN}.  Moreover, we use the graph structure to marginalize variables that do not appear in the error terms. The idea is to split the sum $\sum_{\zd \in Z(A)}$ into $\sum_{\zd_{[1:x_1]}} \sum_{\zd_{(x_1:x_2]}} \ldots \sum_{\zd_{(x_k:N]}} $, where $\{x_1, x_2 \ldots , x_k \}$ denotes all the indices in $\cC(A)$ that show up as $\Delta \cP{}{}{}$ in the expression for $\cL$. 
	\begin{align*}
		\sum_{\zd \in Z(A)} |\cL| &\leq \sum_{\zd \in Z (A)} \sum_{k = 2}^{|\cC (A)|} \sum_{\substack{U \subseteq \cC(A) \\ |U| = k}} \left(\prod_{C_i \in U }  \left| \Delta \cP{\zd_{C_i}}{\zd_{\pa(C_i)}}{A}\right| \right) \left(\prod_{C_j \in \cC(A) \setminus U} \cP{\zd_{C_j}}{\zd_{\pa(C_j)}}{A} \right)\\
		&=  \sum_{k = 2}^{|\cC (A)|} \sum_{\zd \in Z (A)}  \sum_{\substack{U \subseteq \cC(A) \\ |U| = k}} \left(\prod_{C_i \in U }  \left| \Delta \cP{\zd_{C_i}}{\zd_{\pa(C_i)}}{A}\right| \right) \left(\prod_{C_j \in \cC(A) \setminus U} \cP{\zd_{C_j}}{\zd_{\pa(C_j)}}{A} \right) \\
   &\leq \sum_{k = 2}^{|\cC (A)|} \sum_{\substack{U \subseteq \cC(A) \\ |U| = k}} \sum_{\zd \in Z (A)} \eta^k   \left(\prod_{C_j \in \cC(A) \setminus U} \cP{\zd_{C_j}}{\zd_{\pa(C_j)}}{A} \right) \nonumber \tag{via Lemma \ref{lem:er_bound SMBN},  $\left|\Delta \cP{\zd_{C_i}}{\zd_{\pa(C_i)}}{A} \right| \leq \eta$ }
	\end{align*}

	First, we upper bound each term considered in the outer-most sum. Towards this, let $U$ denote the set of c-components that show up as $\Delta \cP{}{}{}$,  we define $X \coloneqq \cup_{C_i \in U} C_i = \left\{ x_1, x_2 , \cdots , x_m \right\}$ where $x_i$ denotes the vertex $V_{x_i} \in \cV(A)$. Note that since c-components are at most of size $\ell$ and for $|U| = k$, we have $|X| \leq \ell k$. Now, using Lemma \ref{lem:pseudo_pa}, we obtain 
  \begin{align}
  &\sum_{\substack{U \subseteq \cC(A) \\ |U| = k}} \sum_{\zd \in Z (A)} \eta^k   \left(\prod_{C_j \in \cC(A) \setminus U} \cP{\zd_{C_j}}{\zd_{\pa(C_j)}}{A} \right) \nonumber \\
  &=\sum_{\substack{U \subseteq \cC(A) \\ |U| = k}} \sum_{\zd \in Z (A)} \eta^k   \left(\prod_{j \in \cV(A) \setminus X} \cP{\zd_j}{\zd_{\qa(j)}}{A} \right) \nonumber \\
		 & =\sum_{\substack{U \subseteq \cC(A) \\ |U| = k}} \eta ^k 
		\sum_{\zd_{[1: \x_1]} \in Z_{[1: \x_1]}(A)} 
		\left(\prod_{j \in  \cV_{[1:x_1) }(A) } \cP{\zd_j}{\zd_{\qa(j)}}{A} \right) \sum_{\zd_{(x_1: \x_2]} \in Z_{(x_1: x_1]}(A)} \left(\prod_{j \in  \cV_{(x_1:x_2) }(A) } \cP{\zd_j}{\zd_{\qa(j)}}{A} \right)  \nonumber
		\ldots
		 \\& \sum_{\zd\in Z_{(x_i: x_{i+1}]}(A)} \left(\prod_{i \in \cV_{(x_i:x_{i+1}) }(A) } \cP{\zd_j}{\zd_{\qa(j)}}{A} \right) \ldots \sum_{\zd_{(x_k : N ]} \in Z_{(x_k: N]}(A)} \left(\prod_{j \in \cV_{(x_k:N] }(A) } \cP{\zd_j}{\zd_{\qa(j)}}{A} \right) \label{eq:sec SMBN}
	\end{align}

	The last term in the above expression can be bounded as follows
	\begin{align*}
		\sum_{\zd_{(x_k: N]} \in Z_{(x_k: N]}(A)} \left(\prod_{i \in \cV_{(x_k:N] }(A) } \cP{\zd_j}{\zd_{\qa(j)}}{A} \right) &= \sum_{\zd_{(x_k: N]} \in Z_{(x_k: N]}(A)} \prob_{ \Do(A)} \left[ \cV_{(x_k: N]}(A) = \zd_{(x_k: N]} | \qa (\cV_{(x_k: N]}(A)) \right]\\
  &=\prob_{ \Do(A)} \left[ V_N = 1 | \qa \left( \cV_{(x_k:N] }(A) \right)\right] \leq 1.
	\end{align*}
	For all the other terms, we have the following bound   
	\begin{align*}
		\sum_{\zd\in Z_{(x_i: x_{i+1}]}(A)}& \left(\prod_{i \in \cV_{(x_i:x_{i+1}) }(A) } \cP{\zd_j}{\zd_{\qa(j)}}{A} \right) 
		\\&= \sum_{\zd_{x_{i+1}} \in \{0,1\}} \sum_{\zd_{(x_i: x_{i+1})} \in Z_{(x_i: x_{i+1})}(A)} \left(\prod_{i \in \cV_{(x_i:x_{i+1}) }(A) } \cP{\zd_j}{\zd_{\qa(j)}}{A} \right)
		\\ &= \sum_{\zd_{x_{i+1}} \in \{0,1\}} \sum_{\zd_{(x_i: x_{i+1})} \in Z_{(x_i: x_{i+1})}(A)} \prob_{ \Do(A)} \left[ \cV_{(x_i:x_{i+1}) }(A) = \zd_{(x_i: x_{i+1})} | \qa \left( \cV_{(x_i:x_{i+1}) }(A) \right)\right]
  \\ & \leq \sum_{\zd_{x_{i+1}} \in \{0,1\}} 1  \\&=  2.
	\end{align*}

Substituting in (\ref{eq:sec SMBN}), we get
\begin{align*}
\sum_{\substack{U \subseteq \cC(A) \\ |U| = k}} \sum_{\zd \in Z (A)} \eta^k   \left(\prod_{C_j \in \cC(A) \setminus U} \cP{\zd_{C_j}}{\zd_{\pa(C_j)}}{A} \right) &\leq \sum_{\substack{U \subseteq \cC(A) \\ |U|= k}} \eta^k \ 2^{\ell k}  \tag{since $|X| \leq \ell k$} \\  
 &= \binom{N}{k} \left( 2^\ell \eta\right)^k.
\end{align*}

Therefore, the sum  $\sum_{\zd \in Z(A)} |\cL|$ satisfies

\begin{align*}
	\sum_{\zd \in Z(A)} |\cL| &\leq \sum_{k=2}^{N}\binom{N}{k} \left( 2^\ell \eta\right) ^k \\
	&= \sum_{k=0}^{N}\binom{N}{k} \left( 2^\ell \eta\right) ^k  - 2^\ell N\eta - 1 \\
	&=(1+2^\ell \eta)^N-2^\ell N\eta-1\\
	&\leq e^{2^\ell N\eta}-2^\ell N\eta -1\\
	&\leq 1 + 2^\ell N\eta + (2^\ell N\eta)^2 - 2^\ell N \eta -1 \tag{with $\eta \leq \frac{1}{2^\ell N}$}\\
	&\leq 4^\ell N^2\eta^2.
\end{align*}
The lemma stands proved.
\end{proof}
\begin{lemma}\label{lem:h-bound-smbn}
For estimates obtained via a covering intervention set $\cI$, the following event holds with probability at least $1-\frac{2}{T}$: 
\begin{equation*}
  	\left| \sum_{\zd \in Z(A)}\cH \right| \leq \sqrt{\frac{N \ 4^\ell  \ 2^d \ |\cI|  \log{(|\cA| T)}}{T}}  \qquad \text{for all $A \in \cA$.} 
\end{equation*} 
\end{lemma}
\begin{proof}
Equation (\ref{eqn:SMBN-exp-Hz}) gives us 
	\begin{align*}
		\left| \sum_{\zd \in Z(A)}\cH \right|
         = \left| \sum_{ C_i \in \cC (A)} \sum_{\zd \in Z(A)}  \Delta \cP{\zd_{C_i}}{\zd_{\pa(C_i)}}{A} \prod_{ \substack{j\in \cV(A) \setminus C_i}}  \cP{\zd_{j}}{\zd_{\qa(j)}}{A} \right|.
    \end{align*}
Let $X \coloneqq \left\{ x_1, x_2 \cdots x_m \right\}$ be the vertices in a c-component $C_i$ considered in the outer summation. Furthermore, for ease of exposition, write $(x_k:x_{k+1})'  \coloneqq (x_k:x_{k+1}) \setminus \pa(C_i)$, i.e., the set $(x_k:x_{k+1})'$ excludes the parents of the c-component $C_i$. We have 
\begin{align*}
		&\left| \sum_{\zd \in Z(A)}\cH \right| \\&= \Biggl|  \sum_{ C_i \in \cC (A)}\sum_{\substack{\zd_{\pa(C_i)} \in \\  Z_{\pa(C_i)}(A)}} \sum_{\substack{\zd_{C_i} \in \\  Z_{C_i}(A)}} \Delta \cP{\zd_{C_i}}{\zd_{\pa(C_i)}}{A} 
    \sum_{\substack{\zd_{[1 : x_1)' } \in \\ Z_{[1 : x_1 )'}(A)}} \prod_{ j_1 \in \cV_{[1:x_1)}(A)} \cP{\zd_{j_1}}{\zd_{\qa(j_1)}}{} \\ 
    &\sum_{\substack{\zd_{(x_1: x_2 )'} \in \\ Z_{(x_1: x_2)'}(A)}} \prod_{ j_2  \in \cV_{(x_1:x_2)}(A)}  \cP{\zd_{j_2}}{\zd_{\qa(j_2)}}{A} \ldots
    \sum_{\substack{\zd_{(x_k: x_k+1 )'} \in \\ Z_{(x_k: x_{k+1})'}(A)}} \prod_{ j_2  \in \cV_{(x_k:x_{k+1})}(A)}  \cP{\zd_{j_k}}{\zd_{\qa(j_k)}}{A} \ldots \Biggr| \\
    &=\left|\sum_{ C_i \in \cC (A)}\sum_{\substack{\zd_{\pa(C_i)} \in \\  Z_{\pa(C_i)}(A)}} \sum_{\substack{\zd_{C_i} \in \\  Z_{C_i}(A)}} \Delta \cP{\zd_{C_i}}{\zd_{\pa(C_i)}}{A} c_i\left( \zd_{C_i} , \zd_{\pa(C_i)} \right) \right|.
\end{align*}
Here, 
 \begin{align*}
 c_i&(\zd_{C_i}, \zd_{\pa (C_i)}) \coloneqq \\ & \sum_{\substack{\zd_{[1 : x_1)' } \in \\ Z_{[1 : x_1 )'}(A)}  } \prod_{ j_1 \in \cV_{[1:x_1)}(A)} \cP{\zd_j}{\zd_{\qa(j_1)}}{A} 
    \sum_{\substack{\zd_{(x_1: x_2 )'} \in \\ Z_{(x_1: x_2)'}(A)}} \prod_{ j_2  \in \cV_{(x_1:x_2)}(A)}  \cP{\zd_{j_2}}{\zd_{\qa(j_2)}}{A} \cdots \\
    &\sum_{\substack{\zd_{(x_k: x_k+1 )'} \in \\ Z_{(x_k: x_{k+1})'}(A) }} \prod_{ j_k  \in \cV_{(x_k:x_{k+1})}(A)}  \cP{\zd_{j_k}}{\zd_{\qa(j_k)}}{A} \cdots \sum_{\substack{\zd_{(x_m: N ]'} \in \\ Z_{(x_m: N]' }(A) }} \prod_{ j_m  \in \cV_{(x_k:x_{k+1})}(A)}  \cP{\zd_{j_m}}{\zd_{\qa(j_m)}}{A}.
\end{align*}

We show in Claim \ref{lem:bound_c_smbn} (proved below) that $c_i(z_{C_i}, \zd_{\pa (C_i)}) \leq 1$. Therefore, 
\begin{align}
\left| \sum_{\zd \in Z(A)}\cH \right| 
  &\leq  \left| \sum_{ C_i \in \cC (A) } 
  \sum_{\substack{\zd_{\pa(C_i)} \in \\  Z_{\pa(C_i)}(A)}} \sum_{\substack{\zd_{C_i} \in \\  Z_{C_i}(A)}} \Delta \cP{\zd_{C_i}}{\zd_{\pa(C_i)}}{}) \right| \label{ineq:A} 
\end{align}

Since $\cI$ is a covering intervention set, for each pair $(C_i,\zd_{\pa(C_i)})$, there exits an intervention $I \in \cI $ such that intervening $\Do(I)$ provides a sample for the distribution $\prob [ \cV_{C_i} \mid \Do(\pa(C_i) = \zd_{\pa(C_i)}) ]$. Hence, we have at least $\frac{T}{|\cI|}$ samples for the distribution $\prob [ \cV_{C_i} \mid \Do(\pa(C_i) = \zd_{\pa(C_i)}) ]$. We represent the $s^{th}$ sample for the distribution by indicator random variable $Y_{s}(\zd_{C_i}, \zd_{\pa(C_i)})$ which takes value one when $ \cV_{C_i} = \zd_{C_i} $, else its zero. Hence, inequality (\ref{ineq:A}) reduces to 
\begin{align*}
	\left| \sum_{\zd \in Z(A)}\cH \right|  \leq \left| \sum_{ C_i \in \cV (A) } \sum_{\substack{\zd_{\pa(C_i)} \in \\  Z_{\pa(C_i)}(A)}} \frac{|\cI|}{T} \sum_{s=1}^{T/|\cI|} \left( \sum_{ \zd_{C_i} \in Z_{C_i}(A)} Y_s(\zd_{C_i},\zd_{\pa(C_i)}) - \cP{\zd_{C_i}}{\zd_{\pa(C_i)}}{A}  \right)  \right|.
\end{align*}

In the above expression, the term $\sum_{ \zd_{C_i} \in Z_{C_i}(A)} Y_s(\zd_{C_i},\zd_{\pa(C_i)}) - \cP{\zd_{C_i}}{\zd_{\pa(C_i)}}{A}$ is an independent random quantity bounded between $[-2^{|C_i|} ,2^{|C_i|}] $.
We now apply Heoffding's inequality (Lemma \ref{lem:hoeff})
\begin{align*}
	\prob_{ \Do(A)} \left[ \left| \sum_{\zd \in Z(A)}\cH \right| \geq \varepsilon \right] \leq 2 \mathrm{exp}\left( \frac{ - T \varepsilon^2}{ 2 |\cI| \sum_{ C_i \in \cC (A) } \sum_{\zd_{\pa (i)} \in \zd_{\pa(i)}} 2^{2|C_i|}} \right) \\  \label{ineq:heof} 
   \leq 2 \mathrm{exp}\left( \frac{ - T \varepsilon^2}{ 2 |\cI| \sum_{ C_i \in \cC (A) } \sum_{\zd_{\pa (i)} \in \zd_{\pa(i)}} 2^{2\ell}} \right)
   \leq 2 \mathrm{exp}\left( \frac{ - T \varepsilon^2}{ 2 |\cI| N 2^{\ell d} \cdot 2^{2 \ell}} \right).
\end{align*}
Setting $\varepsilon = \sqrt{\frac{2 N \ |\cI| \ 2^{\ell d} \ 4^{\ell} \log{\left(|\cA| \cdot T \right)}}{T}}$ and taking union bound over all of $A \in \cA$, gives us the required probability bound. This completes the proof of the lemma.
\end{proof}

We next establish the claim used in the proof of Lemma \ref{lem:h-bound-smbn}.
	
\begin{claim}\label{lem:bound_c_smbn}
$$c_i(\zd_{C_i}, \zd_{\pa (C_i)}) \leq 1.$$
\end{claim}
\begin{proof}
    It holds that 
    \begin{align*}
 c_i&(\zd_{C_i}, \zd_{\pa (C_i)}) =\\ & \sum_{\substack{\zd_{[1 : x_1)' } \in \\ Z_{[1 : x_1 )'}(A)}  } \prod_{ j_1 \in \cV_{[1:x_1)}(A)} \cP{\zd_j}{\zd_{\qa(j)}}{A} 
    \sum_{\substack{\zd_{(x_1: x_2 )'} \in \\ Z_{(x_1: x_2)'}(A)}} \prod_{ j_2  \in \cV_{(x_1:x_2)}(A)}  \cP{\zd_{j_2}}{\zd_{\qa(j_2)}}{A} \cdots \\
    &\sum_{\substack{\zd_{(x_k: x_k+1 )'} \in \\ Z_{(x_k: x_{k+1})'}(A) }} \prod_{ j_k  \in \cV_{(x_k:x_{k+1})}(A)}  \cP{\zd_{j_k}}{\zd_{\qa(j_k)}}{A} \cdots \sum_{\substack{\zd_{(x_m: N ]'} \in \\ Z_{(x_m: N]' }(A) }} \prod_{ j_k  \in \cV_{(x_k:x_{k+1})}(A)}  \cP{\zd_{j_k}}{\zd_{\qa(j_k)}}{A}
    \end{align*}
   We can upper bound each term in the above expression as shown below,
    \begin{align*}
        &\sum_{\substack{\zd_{(x_k: x_k+1 )'} \in \\ Z_{(x_k: x_{k+1})'}(A) }} \prod_{ j_k  \in \cV_{(x_k:x_{k+1})}(A)}  \cP{\zd_{j_k}}{\zd_{\qa(j_k)}}{A} \\
        & =\sum_{\substack{\zd_{(x_k: x_k+1 )'} \in \\ Z_{(x_k: x_{k+1})}(A) }} \prob_{\Do(A)} \left[ \cV_{(x_k:x_{k+1})}(A) = \zd_{(x_k: x_k+1 )} | \pa'(x_k:x_{k+1})\right]\\
        &= \prob_{\Do(A)} \left[ \cV_{(x_k:x_{k+1}) \cap \pa(C_i)}(A) = \zd_{(x_k: x_k+1 )\cap \pa(C_i)} | \pa'(x_k:x_{k+1})\right]\\
        &\leq 1.
    \end{align*}
    Substituting this in the expression for $c_i(\zd_{C_i}, \zd_{\pa (C_i)})$, we get the required bound. 
\end{proof}


\subsection{Proof of Theorem \ref{thm: main theorem SMBN}}

Lemma \ref{lem:covers SMBN} implies that, with probability at least $\left(1 - \frac{1}{T}\right)$, the set $\cI$ is indeed a covering intervention set for the graph $\cG$. We combine this guarantee with Lemmas \ref{lem:l-bound-smbn} and \ref{lem:h-bound-smbn}. In particular, with probability at least $\left(1- \frac{5}{T}\right)$, we have, for all $A \in \cA$:
\begin{align*}
        \left|\mu(A) - \muh(A)\right|   &= \left|\sum_{\zd \in Z(A)} \left( \mathcal{H}_\zd + \mathcal{L}_\zd \right) \right|\\
        & \leq \sqrt{\frac{N  ~ 4^\ell  ~ 2^d ~ |\cI|  \log{(|\cA| T)}}{T}} + \frac{4^\ell  N^2 |\cI| (\ell d+ \ell + \log{(NT)})}{T}\\
        &\leq 2\sqrt{\frac{N ~ 4^\ell  ~ 2^d ~ |\cI| \log(|\cA| T)}{T}} \tag{For $T \gtrsim  N^3$}
\end{align*}
Let $A_T$ be the output after $T$ rounds of interventions, i.e.,  $A_T = \argmax_{A\in \cA} \ \muh(A)$. In addition, let $A^* = \argmax_{A\in \cA} \ \mu(A)$ be the optimal intervention. Hence, with probability at least $1-\frac{5}{T}$ we have, 
    \begin{align}
        \mu(A^*)- \mu(A_T) \leq 4  \sqrt{\frac{N ~ 4^\ell  ~ 2^d ~ |\cI| \log(|\cA| T)}{T}}
    \end{align}
This gives the desired upper bound on the simple regret, $R_T$:
\begin{align*}
    R_T = \E \left[ \mu(A^*)- \mu(A_T) \right] \leq  \left( 4  \sqrt{\frac{N ~ 4^\ell  ~ 2^d ~ |\cI| \log(|\cA| T)}{T}}\right) \left(1-\frac{5}{T}\right) + \frac{5}{T} 
    \leq 5  \sqrt{\frac{N ~ 4^\ell  ~ 2^d ~ |\cI| \log(|\cA| T)}{T}}.
\end{align*}
For SMBNs, since the size of the covering intervention set  satisfies 
$|\cI| = (3d)^\ell \cdot  2^{\ell d} (\log N +2\ell d + \log T )$ (see Lemma \ref{lem:covers SMBN}), we also have the following explicit form of the simple regret bound
\begin{align*}
		\Rg_T = O \left( \sqrt{ \frac{N \  (3  d \ 8^d )^\ell  \  \log{|\cA | } }{T}  }\log{T} \right).
	\end{align*}
The theorem stands proved.

\section{Experiments}
\label{sec: experimental details}

\begin{figure}[!thb]
\centering
\begin{minipage}[c]{.9\textwidth}
  \centering
\includegraphics[scale=0.55]{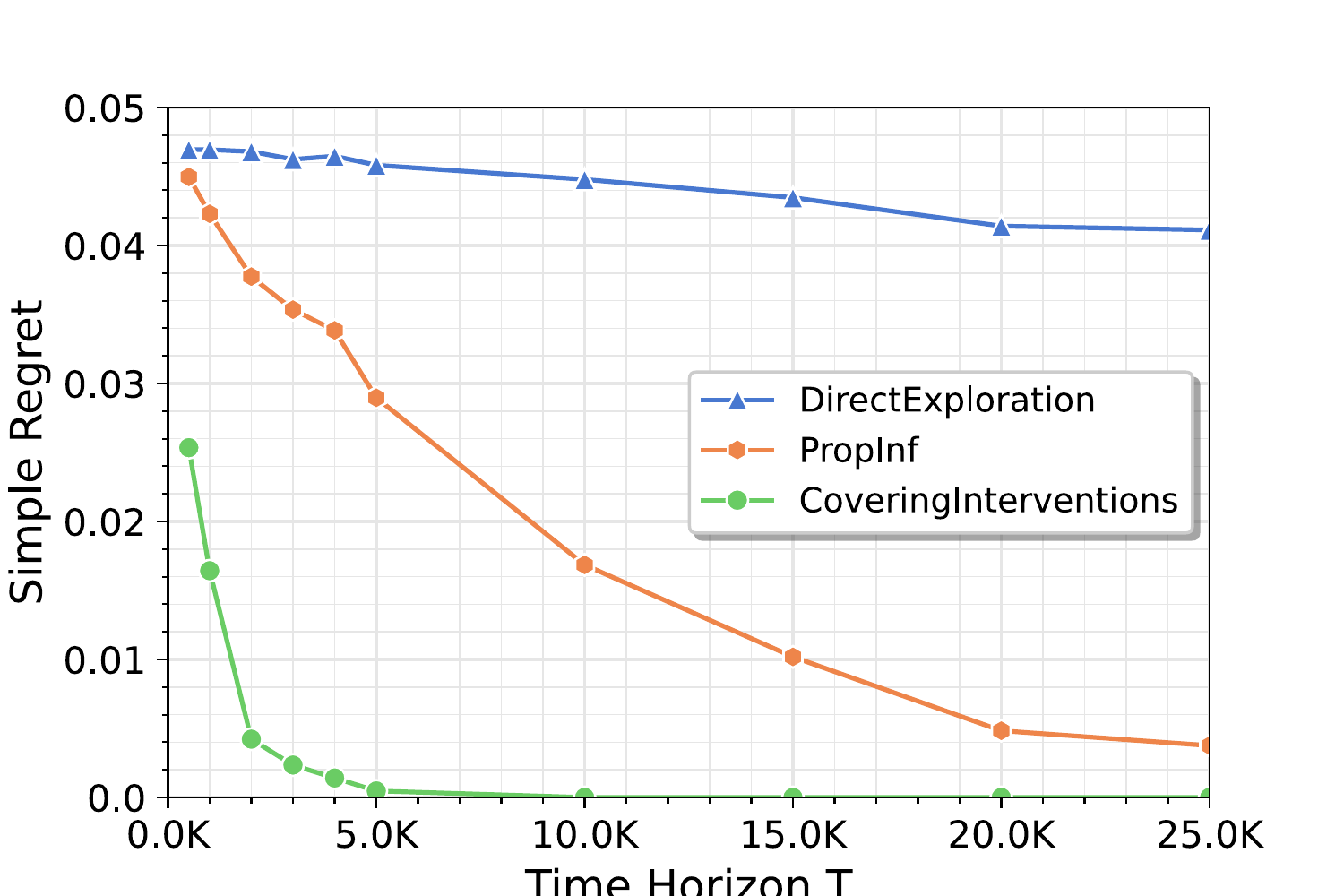} 
  \captionof{figure}{Plot of simple regret with rounds of exploration.}
  \label{fig:Plot of Average regret with number of iterations}
\end{minipage}
\end{figure}

This section provides empirical evaluations for our algorithm. We use a setup similar to one considered in \cite{yabe2018causal}. In the experiments below, we compare our algorithm, $\CI$ (Algorithm \ref{algo: main algo fully obs setting}), with two baseline methods: 

\noindent 
(i) $\DE$ that samples for each intervention $A \in \cA$ and selects the one that yields the maximum empirical mean for the reward variable.

\noindent 
(ii) $\PI$, the propagating inference algorithm of Yabe et al.~\cite{yabe2018causal}. {As in implementation of \cite{yabe2018causal} (see Section 5 of the cited paper), we uniformly sample and do not explicitly solve their proposed optimization problem.}

For our experiments, we set the causal graph $\cG$ (over Bernoulli random variables) to be a directed complete binary tree with height $h=7$ and, hence, $128$ leaves. The root note is the designated reward variable and each edge points towards the root. 

Write $L$ to denote a subset of two leaves that connect to the same vertex at height $(h-1)$, i.e., $|L|=2$ and the two leaves in $L$ are parents of a single vertex. We collect all such $2^{h-1}=64$ subsets $L$ to populate the family $\mathcal{L}$. The set of interventions $\cA \coloneqq \{ \Do(L = s) \mid L \in \mathcal{L} \text{ and } s\in \{0,1\}^2 \}$. In our setup, the leaf vertices take value $0$, unless intervened upon.
We fix one intervention $A^*\in\A$ to yield a greater expected reward (at $V_N$) than all other interventions $A \in \cA$. In particular, fix an arbitrary vertex $V_w$ at height $(h-1)$ of the tree $\cG$ and let $L_w$ be the set of $2$ leaves connected to $V_w$. We set intervention $A^* = \Do(L_w = \mathbf{1})$ and $\Prob(V_w = 1 \mid \Do(A^*) ) = \pi + \varepsilon$, with appropriately small parameters $\pi \in (0,1)$ and $\varepsilon \in (0,1)$. For all other interventions $A \in \cA$ and all vertices $V_x$ at height $(h-1)$, we have $\Prob(V_x = 1 \mid \Do(A) ) = \pi$. 
All other internal vertices, including the  reward variable $V_N$, take value equal to the logical \texttt{or} operation of their parents. Note that an algorithm that selects $A^*\in\A$ incurs simple regret $0$. Otherwise, the simple regret is at least $\pi + \varepsilon - \pi\ 2^{h-1}$, which is close to $\varepsilon$ for $\pi \ll 1/2^{h-1}$. \\

\noindent
\textbf{Simple Regret vs.~Time:} In our experiments, we compare the simple regret with time horizon $T$, for the three algorithms.
We run the algorithms 1000 times each and average the simple regret over various runs. 
We use $\pi=0.001$, and $\varepsilon=0.05$. For each of the $2^{h-1}=64$ vertices at height $h-1$, we can set their parents to one of $2^2=4$ assignments. 
Hence, the total number of interventions $|\A| = 64\times 4=256$. 

We plot our results in \hyperref[fig:Plot of Average regret with number of iterations]{Figure \ref{fig:Plot of Average regret with number of iterations}} and show that $\CI$ converges to $0$ regret the fastest.

\section{Conclusion and Future Work}
\label{sec: conclusion and future work}
Using the idea of covering interventions, this paper obtains  improved simple regret guarantees for the causal bandit problem. We also generalize the guarantee to causal graphs with unobserved variables. Notably, and in contrast to prior works, our regret guarantees only depend on the explicit problem parameters. Our experiments empirically highlight that our algorithm  provides improvements over baselines. Establishing lower bounds in the general causal bandit setup is an important direction of future work. It is also interesting to develop computationally efficient (simple regret) algorithms for settings in which the target set $\cA$ is large and implicitly specified.

\bibliographystyle{alpha}
\bibliography{references}
\newpage

\appendix
\section{Missing Proof from Section \ref{section:regret-analysis SMBN}}
\label{appendix:regret-analysis SMBN}
This section provides a proof of Lemma \ref{lem:pseudo_pa}.

\LemmaPseudoPa*

\begin{proof}
    First, note that intervening on parent vertices of a c-component (under intervention $A$) is the same as conditioning on them. Specifically, 
    \begin{align*}
         \cP{\zd_{C}}{}{\zd_{\pa(C)}} = \cP{\zd_{C}}{\zd_{\pa{(C)}}}{A} 
    \end{align*}
    Further, the chain rule of conditional probability gives us 
     \begin{align*}
        \cP{\zd_{C}}{\zd_{\pa{(C)}}}{A} &= \prod_{j \in C} \prob_{\Do(A)} \left[U_j = \zd_{j} \mid \pa(C) = \zd_{\pa(C)}, (U_1 \ldots U_{j-1}) = \zd_{(U_1 \ldots U_{j-1})}\right] 
    \end{align*}
    Next, we use the notion of d-separation (see \cite{pearl2009causality} ) to argue that conditioning on just the set $\qa(j)$ is sufficient. In particular, note that the set $Y = \pa(\{U_{j+1} \ldots U_m\})$ is d-separated from vertex $U_j$ by the set $X = \pa(\{U_{1} \ldots U_j\}) \cup (\{U_{1} \ldots U_{j-1}\}) $. This is due to the fact that all paths from a vertex in $Y$ to $U_j$ are either blocked by a collider vertex in $\{U_{j+1} \ldots U_m\}$ (and the collider vertex is not included $X$), or the path is blocked by a vertex in $X$. This implies that conditioned on $X$, $U_j$ is independent of all vertices in $Y$ \cite{pearl2009causality}. Formally, we write
    \begin{align*}
         &\prob_{\Do(A)} \left[U_j = \zd_{j} \mid \pa(C) = \zd_{\pa(C)}, (U_1 \ldots U_{j-1}) = \zd_{(U_1 \ldots U_{j-1})}\right] \\
         &= \prob_{\Do(A)} [U_j = \zd_{j} \mid \pa({U_1 \ldots U_{j-1}}) = \zd_{\pa(U_1 \ldots U_j)}, \pa({U_{j+1} \ldots U_{m}}) = \zd_{\pa(U_{j+1}\ldots U_m)}, (U_1 \ldots U_{j-1}) = \zd_{(U_1 \ldots U_{j-1})}] \\
        &= \prob_{\Do(A)} \left[U_j = \zd_{j} \mid \pa({U_1 \ldots U_{j-1}}) = \zd_{\pa(U_1 \ldots U_j)}, (U_1 \ldots U_{j-1}) = \zd_{(U_1 \ldots U_{j-1})}\right]  \tag{since $\pa(\{U_1 \ldots U_{j} \}) \cup \{U_1 \ldots U_{j-1} \}$  d-separates $U_j$ from $\pa(\{U_{j+1}\ldots U_m \})$ }\\
        &= \prob_{\Do(A)} \left[U_j = \zd_{j} \mid \qa(j)\right] 
        \tag{by definition of $\qa(j)$}
    \end{align*}
    Therefore, 
     \begin{align*}
        \cP{\zd_{C}}{}{\zd_{\pa(C)}} &= \cP{\zd_{C}}{\zd_{\pa{(C)}}}{A} \\
        &=  \prod_{j \in C} \prob_{\Do(A)}  \left [V_j = \zd_{j} \mid \qa(j) = \zd_{\qa(j)} \right] \\
        &= \prod_{j \in C} \cP{\zd_j}{\zd_{\qa(j)}}{A}
    \end{align*}
This completes the proof of the lemma.
\end{proof}

\end{document}